\documentclass[11pt]{article}

\usepackage[sort,numbers,compress]{natbib}

\usepackage[margin=1in]{geometry}

\usepackage[utf8]{inputenc} 
\usepackage[T1]{fontenc}    
\usepackage{hyperref}       
\usepackage{url}            
\usepackage{booktabs}       
\usepackage{amsfonts}       
\usepackage{nicefrac}       
\usepackage{microtype}      
\usepackage{xcolor}         

\usepackage{wrapfig}  
\usepackage{graphicx}
\usepackage{subcaption}

\usepackage[ruled, noend, linesnumbered]{algorithm2e}
\SetKwRepeat{Do}{do}{while}
\usepackage{amsmath}
\usepackage{amssymb}
\usepackage{amsthm}
\usepackage{hyperref}
\usepackage{enumitem}
\usepackage[capitalise]{cleveref}
\usepackage{csquotes}

\newtheorem{lemma}{Lemma}[section]
\newtheorem{theorem}[lemma]{Theorem}

\newtheorem{definition}[lemma]{Definition}
\newtheorem{corollary}[lemma]{Corollary}
\newtheorem{proposition}[lemma]{Proposition}

\newtheorem{observation}{Observation}

\newcommand{\cA}{\mathcal{A}}
\newcommand{\cB}{\mathcal{B}}
\newcommand{\cC}{\mathcal{C}}
\newcommand{\cD}{\mathcal{D}}
\newcommand{\cE}{\mathcal{E}}
\newcommand{\cF}{\mathcal{F}}
\newcommand{\cG}{\mathcal{G}}
\newcommand{\cH}{\mathcal{H}}
\newcommand{\cI}{\mathcal{I}}

\newcommand{\cL}{\mathcal{L}}
\newcommand{\cM}{\mathcal{M}}
\newcommand{\cN}{\mathcal{N}}

\newcommand{\cS}{\mathcal{S}}
\newcommand{\cT}{\mathcal{T}}

\newcommand{\cX}{\mathcal{X}}

\newcommand{\cZ}{\mathcal{Z}}

\newcommand{\opt}{\mathrm{OPT}}

\newcommand{\E}{\mathbb{E}}
\newcommand{\R}{\mathbb{R}}
\newcommand{\N}{\mathbb{N}}

\newcommand{\bfo}{\mathbf{1}}

\DeclareMathOperator{\poly}{poly}
\DeclareMathOperator{\err}{error}

\DeclareMathOperator{\cost}{{\sf Cost}}

\DeclareMathOperator{\regret}{{\sf Regret}}
\DeclareMathOperator{\last}{{\sf Last}}
\DeclareMathOperator{\new}{{\sf New}}
\DeclareMathOperator{\true}{{\sf True}}
\DeclareMathOperator{\false}{{\sf False}}

\newcommand*{\tl}[1]{\textcolor{blue}{[{\sf TL:} #1}]}
\newcommand*{\gb}[1]{\textcolor{brown}{[{\sf GB:} #1]}}
\newcommand*{\rd}[1]{\textcolor{red}{[{\sf RD:} #1}]}
\renewcommand{\tl}[1]{}
\renewcommand{\gb}[1]{}
\renewcommand{\rd}[1]{}

\newcommand{\B}{\mathcal{B}}
\newcommand{\I}{\mathcal{I}}
\renewcommand{\cS}{\mathcal{S}}
\newcommand{\bone}{\textsf{1}}

\title{Offline Local Search for Online Stochastic Bandits}

\author{Gerdus Benad\`{e}\thanks{Boston University, MA, USA. \texttt{benade@bu.edu}.} 
\and Rathish Das\thanks{University of Houston, TX, USA. \texttt{rathish@uh.edu}.} 
\and Thomas Lavastida\thanks{University of Texas at Dallas, TX, USA. \texttt{thomas.lavastida@utdallas.edu}.}}

\date{}

\begin{document}

\maketitle

\begin{abstract} 
    Combinatorial multi-armed bandits provide a fundamental online decision-making environment where a decision-maker interacts with an environment across $T$ time steps, each time selecting an action and learning the cost of that action. 
    The goal is to minimize regret, defined as the loss compared to the optimal fixed action in hindsight under full-information. 
    There has been substantial interest in leveraging what is known about offline algorithm design in this online setting. 
    Offline greedy and  linear optimization algorithms (both exact and approximate) have been shown to provide useful guarantees when deployed online. 
    We investigate local search methods, a broad class of  algorithms used widely in both theory and practice, which have thus far been under-explored in this context. 
    We focus on problems where offline local search terminates in an approximately optimal solution and give a generic method for converting such an offline algorithm into an online   stochastic combinatorial bandit algorithm with  $O(\log^3 T)$ (approximate) regret.  
    In contrast, existing  offline-to-online frameworks yield regret (and approximate regret) which depend sub-linearly, but polynomially on $T$. 
    We demonstrate the flexibility of our framework by applying it to three  online stochastic combinatorial optimization problems: scheduling to minimize total completion time, finding a minimum cost base of a matroid and uncertain clustering.
\end{abstract} 


\section{Introduction} \label{sec:intro}

Online learning is a cornerstone of modern algorithm design. 
The bandit setting, in which the learner only receives cost feedback,  epitomizes the trade-off between exploration and exploitation \citep{DBLP:journals/ml/AuerCF02, DBLP:journals/siamcomp/AuerCFS02, LaiRobbins85}.
Here,  an algorithm makes a sequence of decisions $x_1, x_2, \ldots, x_T$ over a series of $T$ discrete time periods. 
In response to each decision $x_t$, the environment generates a cost $c_t(x_t)$ (either adversarially~\citep{DBLP:journals/siamcomp/AuerCFS02} or stochastically~\citep{DBLP:journals/ml/AuerCF02}), and the algorithm's objective is to minimize its cumulative cost across all $T$ periods. 
As feedback to improve future decisions, the algorithm observes each cost $c_t(x_t)$, which is known as \emph{bandit feedback}.
Notably, the algorithm does not observe what the cost would have been if, instead, a different action was chosen at each period, necessitating a trade-off between exploring the action space to find promising solutions and exploiting previously discovered low-cost   solutions.
\emph{Regret}  measures  the difference between the cumulative cost achieved by an algorithm and that of the best fixed solution under full information. 
A common objective is to construct algorithms with regret that is sub-linear in $T$, i.e., $o(T)$. 

In the combinatorial bandit setting \citep{DBLP:journals/mor/AudibertBL14, DBLP:conf/icml/ChenWY13, DBLP:conf/nips/CombesSPL15},  
the learner must additionally navigate an exponentially large space of actions, rendering some approaches computationally infeasible.
There is substantial interest in developing general frameworks for designing online  algorithms, 
especially frameworks that  convert offline algorithms (in either a black-  or white-box fashion) into   effective bandit algorithms; see, for example, \citet{DBLP:journals/jcss/KalaiV05, DBLP:journals/siamcomp/KakadeKL09, DBLP:journals/jacm/DuditeHLSSV20}, \citet{NiazadehGWSB_MS23}, and \citet{DBLP:conf/focs/AgarwalGN24}.
Such frameworks allow us to leverage the wealth  of knowledge about offline algorithms  
to construct online learning algorithms.  
Notable techniques for designing offline algorithms that have proven to be successful in  `offline-to-online' frameworks  include greedy algorithms \citep{FouratiQAA24, NiazadehGWSB_MS23, PerraultPV19} and offline linear optimization in both exact~\citep{DBLP:journals/jcss/KalaiV05} and approximate~\citep{DBLP:journals/jcss/KalaiV05, DBLP:journals/siamcomp/KakadeKL09} settings.
These all utilize sophisticated algorithm design and analysis techniques to convert the offline algorithm to an effective bandit learning algorithm with regret (or $\gamma$-regret~\citep{DBLP:journals/siamcomp/KakadeKL09,glasgow2023tight, NiazadehGWSB_MS23}) which scales sub-linearly in $T$, more specifically, with regret which scales  as $T^\rho$  for  some $\rho \in (0,1)$ (usually $\rho = 1/2$ or $\rho = 2/3$). 

The focus of our paper is local search ---  another technique  widely used in the design of offline algorithms.  
A local search algorithm consists of three components: a set of feasible solutions $\cX$, a neighborhood map $\cN: \cX \to 2^{\cX}$, and a cost function $\cost:\cX \to \R$.
Starting from an initial feasible solution $x_0 \in \cX$, local search methods iteratively set $x_{t+1} = \arg\min_{x' \in \cN(x_t)} \cost(x')$, for $t = 1,2,\ldots$ (for a minimization problem), continuing until a locally optimal solution, i.e., a solution $x \in \cX$ such that $\cost(x) \leq \cost(x')$ for all $x' \in \cN(x)$, is found.
Often, to increase efficiency, a local move is accepted only if it improves upon the current solution by a $\beta$-factor, for some $\beta \in (0,1)$,  
and we instead converge to an approximate locally optimal solution.
This method is more formally described in \cref{alg:approx_local_search}.
By setting $\beta$ sufficiently close to 1, the resulting local search method can often be shown to converge in a polynomial number of iterations.   
 
\begin{algorithm}[t] 
    \caption{Generic Local Search\label{alg:approx_local_search}}
    \SetKwProg{proc}{Procedure}{:}{}
    \SetKwComment{Comment}{//}{}
    \SetKwFunction{local}{LocalSearch}
    \SetKw{Break}{break}
    \DontPrintSemicolon
    \KwData{Feasible set $\cX$, neighborhood map $\cN$, cost function $\cost$, and parameter $\beta \in (0,1)$}
    \KwResult{Approximate locally optimal solution}
    \proc{\local{$\cX$, $\cN$, $\cost$}}{
        $x_0 \gets$ any feasible solution from $\cX$\;
        $t \gets 0$\;
        \While{$\true$}{
            $x_{t+1} \gets \arg\min_{x' \in \cN(x_t)} \cost(x')$\;
            \eIf{$\cost(x_{t+1}) \leq \beta \cost(x_t)$}{
                $t \gets t+1$ \Comment{Move to the next iteration}
            }{
            \Return $x_t$ \Comment{Approximate local optimum found}
            }
        }        
    }
\end{algorithm}

In practice, local search methods are appealing due to their ease of implementation, efficiency, and ability to find good solutions quickly~\citep{benoist2011randomized,musliu2004local,ceschia2013local}.
Theoretical approximation guarantees  have also been established for local search methods for a variety of problems \cite{GharanT14,IwamaMY07, BogdanSS013, Chen08, AhmadianFS13, Cohen-AddadGHOS22, arya2001local}.
Despite this widespread use in offline algorithm design, local search has been under-explored as a framework for designing combinatorial bandit algorithms. 
In fact, as far as we are aware, the connection between local search and bandit algorithms has been in the opposite direction ---
several papers use multi-armed bandit algorithms to improve the performance of local search algorithms in practical settings~\citep{DBLP:journals/eor/LagosP24, DBLP:conf/pkdd/YuKM17, DBLP:conf/ecai/Mengshoel0Z20, DBLP:conf/ijcai/Zheng0Z0LM22}.
Thus, the main question we study in this paper is:  

\begin{displayquote}
    \emph{
    Can we convert a good offline local search algorithm into an effective online bandit learning algorithm?  
    If so, what regret guarantees can be achieved using this framework?
    }
\end{displayquote}
 
\subsection{Our Approach and Results} \label{sec:contrib}
Our main contribution is the development of a general framework for converting an offline local search algorithm to an online stochastic bandit algorithm.   
It requires that the local search neighborhood must admit $(\beta, \gamma)$-improving moves which,  informally, means that the neighborhood of a solution $x$ with $\cost(x) > \gamma \cost(x^*)$,  where $x^*$ is an optimal solution, must contain a solution $x'$ which is a $\beta$-factor improvement over $x$, in other words, $\cost(x') \leq \beta \cost(x)$. In particular, it is straightforward to verify that any problem on which \Cref{alg:approx_local_search} terminates in an approximately optimal solution permits $(\beta, \gamma)$-improving moves for appropriate choices of $\beta$ and $\gamma$.
Our main result is that local search algorithms with $(\beta, \gamma)$-improving neighbourhoods can be used online to guarantee $\gamma$-regret with $O(\log^3{T})$ dependence on $T$.

\begin{theorem}[Informal version of \Cref{thm:gamma_regret_main_result}] \label{thm:informal}
Suppose a problem  
admits $(\beta, \gamma)$-improving moves. Then we can construct an  algorithm for the stochastic online variant with bandit feedback with  
    \[
    O\left( {M\cdot C_{\max}\cdot \poly(\beta^{-1}) \cdot \log^3 T} \right)  \ \ \gamma\text{-regret after $T$ rounds,}  
    \] 
    where  
    $C_{\max} = \max_{x \in \cX, z \in \cZ} \cost(x, z)$ is an upper bound on solution costs and $M = \max_{x \in \cX} |\cN(x)|$ is the largest neighborhood of a feasible solution. 
\end{theorem}

We remark that $M$, $C_{\max}$, and $\beta^{-1}$ all can be bounded polynomially in problem specific parameters for each of the applications we consider, and are thus independent of $T$. 
Note that $\beta$ may depend on $\gamma$, so the regret bound may deteriorate as we get closer to exact ($\gamma=1$) regret.
Intuitively, more samples are required to distinguish between solutions with objective function value $\opt$ and $\gamma\cdot\opt$. 
The poly-logarithmic dependence on $T$ is perhaps most surprising. 
While direct comparisons between offline-to-online frameworks is difficult, this is 
a substantially better dependence on $T$ than   
the ${O}(T^{2/3})$ $\gamma$-regret guarantee obtained by \citet{NiazadehGWSB_MS23} via offline greedy algorithms, at the cost of  multiplicative factors that depend on the problem and local search neighborhood.

The resulting algorithm, in essence, attempts to mimic \Cref{alg:approx_local_search} while handling the complications that sprout from the fact that only the stochastically realized cost of the solution submitted in each period is observed. 
First, the cost of a solution must be estimated by sampling for several time periods. 
The algorithm attempts to maintain, with high probability, a solution with expected cost below a threshold, which decreases geometrically over time. 
Whenever the estimated cost of the current solution exceeds this threshold, the solutions in its neighborhood is explored. 
Neighboring solutions must be carefully sampled to ensure that even neighbors with very high costs do not incur large regret, at the same time, we need enough samples to identify $\beta$-improving moves with high probability whenever the current solution is not a $\gamma$-approximation of the optimal solution.
This is done by using a successive-elimination style algorithm \citep{DBLP:conf/colt/Even-DarMM02, DBLP:journals/jmlr/Even-DarMM06} on the neighborhood.
When a $\beta$-improving move is found during exploration it becomes the new current solution. This process repeats until time $T$ or a local optimum is found.  

To round out our results, we apply our framework to the repeated online bandit feedback versions of three problems: scheduling to minimize total completion time under stochastic job sizes, finding a minimum cost base in a matroid under stochastic element costs, and uncertain $k$-median clustering.  

\paragraph{Scheduling to Minimize Total Completion Time:}
When scheduling to minimize total completion time, a feasible solution is a permutation $\pi$ of $n$ jobs which schedules job $j\in [n]$ in position $\pi(j)$. In each time step the size $P_j$ of each job $j$ is drawn from a distribution. Given job sizes $P_1, \ldots, P_n$, the cost of solution $\pi$ is $\cost(\pi, P)  = \sum_{j=1}^n (n - \pi(j) + 1) P_j$, the total makespan. 
We show the neighborhood consisting of swapping a single pair of jobs is $(1-\epsilon/n^2, 1+\epsilon)$-improving. As a result, we have an algorithm for minimizing online total completion time with  $(1+\epsilon)$-regret in the order of $O\left({n^{12} \log^3 T}/{\epsilon^4  }  \right)$. 

\paragraph{Minimum Cost Base in a Matroid:}
For the problem of finding a minimum cost base for a matroid, a feasible solution is a base, or maximal independent set, $B$. Given stochastic realizations of element costs $Z$, the cost of base $B$ is $\cost(B) = \sum_{s\in B} Z_s$. We show the neigborhood which consist of bases constructed by adding an arbitrary element to the current solution and removing an element from the resulting circuit is $(1 - \epsilon/(2r), 1+\epsilon)$-improving, where $r$ is the rank of the matroid. This implies an online algorithm with $(1+\epsilon)$-regret in the order of $O\left( { n r^6 \log^3 T}/{\epsilon^4}  \right)$, where $n$ is the size of the ground set. 
A special case of this result establishes a similar regret bound for the problem of finding  minimum cost spanning trees in a graph.  

\paragraph{Uncertain $k$-Median Clustering:}
In the uncertain $k$-median clustering problem the locations of $n$ points are sampled in a metric space with diameter 1  and the task is to select $k$ cluster  centers from a set of $m$ potential centers to minimize the cost, defined as the sum of distances from each point to its closest cluster center. We use a local search algorithm analysed by \citet{arya2001local} and \citet{cormode2008approximation} to show the single swap neighbourhood is $(1 - 1/n^2, 5(1-1/n))$-improving, implying $5(1-1/n))$-regret which is $O\left({n^{9}m^2\log^3 T}\right)$.   

\subsection{Additional Related Work} \label{sec:rel_work}
 As discussed above, we study a combinatorial bandit setting \citep{DBLP:journals/mor/AudibertBL14, DBLP:conf/icml/ChenWY13, DBLP:conf/nips/CombesSPL15}, and our work is closest to other `offline-to-online' frameworks including \citet{DBLP:journals/jcss/KalaiV05} and  \citet{DBLP:journals/jacm/DuditeHLSSV20}, where it is assumed that the offline version of the problem can be easily solved,  and \citet{NiazadehGWSB_MS23}, which considers problems with robust greedy approximation algorithms satisfying a property called \emph{Blackwell reducibility}. 
We also highlight the recent work of \citet{DBLP:conf/focs/AgarwalGN24} which gives a general framework for online learning in monotone stochastic optimization problems achieving $O(\sqrt{T \log T})$ regret against the best approximation algorithm for the offline problem under known distributions.
Notably, their results hold in the \emph{semi-bandit} setting in which the algorithm observes realizations of some (but not all) of the underlying random variables in addition to the realized cost.
Our results hold for the pure bandit setting where only the cost information is observed.
Next, we briefly touch on other related topics.

\paragraph{Logarithmic Regret in Online Learning:}
Due to the existence of strong lower bounds in many online learning settings \citep{LaiRobbins85, DBLP:journals/ml/AuerCF02},  it is  typically necessary to make additional assumptions concerning the environment, or relax the benchmark, in order to achieve  regret bounds improving beyond the typical $O(T^\rho)$ for some $\rho \in (0,1)$.
For example, in online convex optimization, logarithmic regret is achievable if the sequence of convex functions satisfies \emph{strong convexity} \citep{DBLP:journals/ml/HazanAK07, DBLP:conf/nips/Shalev-ShwartzK08}.
Similarly, logarithmic regret may be possible if certain problem-dependent parameters are bounded appropriately. 
For example, \citet{DBLP:conf/nips/XuW21} achieve logarithmic regret in feature-based dynamic pricing whenever the minimum eigenvalue of a problem-dependent matrix is bounded from below,
and \citet{DBLP:journals/ior/VeraBG21} achieve logarithmic regret for contextual bandits with knapsacks whenever the weight of an item is bounded from below. 

In contrast to assuming additional structure, we  give poly-logarithmic regret guarantees against a relaxed benchmark,   so-called $\gamma$-regret.
Our bound is not problem-dependent, and holds for a wide class of instances as long as the local search neighborhood admits $(\beta, \gamma)$-improving moves.
For settings where the corresponding offline problem is NP-hard, it is necessary to allow a multiplicative approximation factor.
our result is especially interesting in this context,  since we achieve $\gamma$-regret that scales as $O(\log^3 T)$ while prior frameworks including those due to \citet{DBLP:journals/siamcomp/KakadeKL09} and \citet{NiazadehGWSB_MS23}  have $\gamma$-regret which scales as $T^{2/3}$.

\paragraph{Efficient Bandit Algorithms:} 
The exploration strategies specified by many standard approaches to bandit problems can be computationally intractable for large action spaces \citep{lattimore2020bandit}. As a result, there is significant interest in developing algorithms with more efficient implementations \citep{NeuB16, WenKA15, DBLP:conf/icml/YangRS0DS22, PerraultPV19, DBLP:journals/pomacs/CuvelierCG21,DBLP:conf/spaa/BenadeDL25}.
Since our algorithm is based on local search and performs exploration locally, it naturally lends itself to a computationally efficient implementation.

\subsection{Roadmap}

We organize the rest of this paper as follows.
\cref{sec:prelim} formally sets up our model and recalls preliminary results we need for our analysis.
Then we describe our `offline-to-online' algorithm utilizing local search and provide a high-level overview of its analysis in \cref{sec:algorithm}.
Following this, \cref{sec:applications} demonstrates the applicability of our framework on the problems we discussed above. 
\cref{sec:conclusion} concludes the paper and discusses potential directions for future work.

\section{Preliminaries}  \label{sec:prelim}

We now formally define the general problem setting we consider --- stochastic combinatorial bandits with local search.
First, each instance is associated with a set $\cX$ of feasible solutions\footnote{We are mainly interested in the case where $\cX$ is  finite (but potentially large), so minimizing a function over $\cX$ is well-defined.} and for each $x \in \cX$ there is a random non-negative cost associated with it.
We model the randomness as a collection of latent variables taking values in some space $\cZ$, for which there is an unknown distribution $\cD \in \Delta(\cZ)$ over the possible values the latent variables can take.
Thus we can model the cost as a function $\cost: \cX \times \cZ \to \R_+$, which induces a distribution on the realized cost for a fixed $x \in \cX$ as $\cost(x, Z)$, where $Z \sim \cD$.
To streamline our notation, we let $\cost(x) := \E_{Z \sim \cD}[\cost(x, Z)]$ be the expected cost of playing $x \in \cX$ and set $\opt := \min_{x \in \cX} \cost(x)$ to be the minimum expected cost of any feasible solution.

In the online bandit setting, we consider making a sequence of decisions $\{x_t\}_{t=1}^T$, where each $x_t \in \cX$, in order to optimize an objective over a time horizon of $T$ periods in the presence of bandit (cost) feedback.
An algorithm (or policy) $\cA$ is a sequence of maps $\{\cA_t: \cH_{t-1} \to \cX\}_{t=1}^T$ which each takes the observed history $H^{t-1}$ up to period $t-1$ and returns a new solution $x_t$ to be used in period $t$
\footnote{We focus on deterministic algorithms. We can extend to randomized algorithms by letting $A_t$ be a map to $\Delta(\cX)$ and taking an additional expectation over the realized distribution of actions.}.
In the bandit setting, a history up to period $t$ is a sequence $\{X_s, Y_s\}_{s=1}^t$, where $Y_s = \cost(X_s, Z_s)$ is the observed random cost at step $s$ for some independent sample $Z_s \sim \cD$. 
Notice $\{Z_s\}_{s=1}^T$, the realizations of randomness that underlying the cost or reward, is unobserved. 
Formally, histories are random variables over $(\cX \times \R_+)^t$.
We let $\cH^t = (\cX \times \R_+)^t$ be the set of all length $t$ sequences over solution-cost pairs.
For simplicity, we let $\cA_1$ be a constant function since no history has been observed at period 1.
Our main interest is giving algorithms with low $\gamma$-regret, which measures the difference in cumulative expected cost between the algorithm and using a $\gamma$-approximate solution for all $T$ periods, which we can now define formally.

\begin{definition}[$\gamma$-Regret] \label{def:gamma_regret_min}
    For an algorithm $\cA$, latent distribution $\cD$, time horizon $T$, and $\gamma \geq 1$, we define 
    \begin{align} \label{eqn:gamma_regret_min}
        \regret_\gamma(\cA, \cD, T) := \sum_{t=1}^T \E\left[ \cost(X_t, Z_t) \right] - \gamma \cdot T \cdot \opt
    \end{align} 
   to be the $T$-period $\gamma$-regret of algorithm $\cA$ on distribution $\cD$, where the expectation is taken with respect to the independent samples $(Z_1,Z_2, \ldots, Z_T) \sim \cD^T$ and any internal randomness utilized by $\cA$.
   We say that algorithm $\cA$ has $\gamma$-regret $R(T)$ if for all distributions $\cD$, $\regret_\gamma(\cA, \cD, T) \leq R(T)$.
\end{definition}

We remark that if $\gamma=1$, then we recover the standard definition of regret used in stochastic multi-armed bandits.
As discussed in \cref{sec:contrib}, our goal is to find algorithms with $\gamma$-regret whose dependence on $T$ is not just $o(T)$  but is $O(\poly(\log T))$, which is $o(T^{\rho})$ for any $\rho \in (0,1)$.

Local search will allow us to avoid some of  the difficulties that arise with exploring the large action spaces that can arise in stochastic combinatorial bandit problems.  
For a feasible solution space $\cX$,   let $\cN: \cX \to 2^\cX$ be its neighborhood map so that for each $x \in \cX$, $\cN(x)$ is a set of neighboring solutions to $x$.
We assume that $M := \max_{x \in \cX} |\cN(x)|$ is a priori known to the algorithm, which holds in all  the examples we consider.
The key property of a neighborhood map that we require to achieve strong $\gamma$-regret bounds is that any solution which has expected cost worse than $\gamma \opt$ has a neighboring solution which has multiplicative improvement to its cost. 
 
\begin{definition}[$(\beta, \gamma)$-improving moves] \label{def:improving_moves}
    Consider the problem specified by feasible set $\cX$, neighborhood map $\cN$, and expected cost function $\cost:\cX \to \R_+$ as defined above.
    We say the problem admits $(\beta, \gamma)$-improving moves if, for any $x \in \cX$ with $\cost(x) > \gamma \opt$, there exists $x' \in \cN(x)$ with $\cost(x') \leq \beta \cost(x)$.
\end{definition}

The parameters $\beta$ and $\gamma$ can depend on the structure of the problem (e.g., the number of jobs in completion time scheduling) and $\beta$ can depend on $\gamma$ (e.g., if $\gamma = 1+\epsilon$, then $\beta$ can depend on $\epsilon$). A priori, it is perhaps unclear how to establish whether a problem permits $(\beta, \gamma)$-improving moves. Fortunately, there is a straightforward correspondence showing that all  problems where \Cref{alg:approx_local_search} terminates in an approximate solution must permit $(\beta, \gamma)$-improving moves.

 \begin{observation} 
      Given instance $(\mathcal{X}, \mathcal{N}, \cost)$, whenever \Cref{alg:approx_local_search} is guaranteed to terminate in a $\gamma$-approximation  for some choice of $\beta$, then $(\mathcal{X}, \mathcal{N}, \cost)$ permits $(\beta, \gamma)$-improving moves. 
 \end{observation}
 \begin{proof}
   We proceed by contradiction. Suppose that \Cref{alg:approx_local_search} with parameter $\beta$ terminates in a $\gamma$-approximation on a problem defined by $(\mathcal{X}, \mathcal{N}, \cost)$,  but the problem does not permit $(\beta, \gamma)$ improving moves. Then there exists solution $x\in \mathcal{X}$ such that $\cost(x) > \gamma \opt$ and $\cost(x') > \beta \cost(x)$ for all  $x'\in \mathcal{N}(x)$. However, when using $x$ as the starting $x_0$ in \cref{alg:approx_local_search}, the algorithm converges in some solution $x^\circ$ with $\cost(x^\circ) \leq \gamma \opt$. 
   Now either $x = x^\circ$, or the algorithm proceeds to a subsequent solution which yields a $\beta$-factor improvement en route to $x^\circ$. In either case, this contradicts $(\mathcal{X}, \mathcal{N}, \cost)$ not permitting $(\beta, \gamma)$-improving moves.
 \end{proof}

Finally, we require the following version of the Chernoff bounds, see, e.g., \citep{Dubhashi_Panconesi_2009} for a reference.

\begin{theorem} \label{thm:gen_chernoff}
Suppose that $X_1,X_2,\ldots, X_N$ are independent random variables in the interval $[0,1]$.  Let $\mu = \frac{1}{N} \sum_i\E[X_i]$ and $\bar{X} = \frac{1}{N} \sum_s X_s$, then for all $\delta \in (0,1)$ and any $\mu_H, \mu_L$ such that $\mu \in [\mu_L, \mu_H]$ we have:
\[
\Pr \left[  \bar{X}  > (1+\delta) \mu_H \right] \leq \exp \left( - \frac{\delta^2 N \mu_H}{3} \right) \quad \text{and} \quad
\Pr \left[  \bar{X}  < (1-\delta) \mu_L \right] \leq \exp \left( - \frac{\delta^2 N \mu_L}{3} \right).
\]
\end{theorem}

\section{Stochastic Combinatorial Bandits from Offline Local Search}
\label{sec:algorithm}

Our algorithm is formally described in \cref{alg:eps_regret,alg:tester}.
The algorithm operates in a sequence of phases $\ell = 1,2,\ldots$ which are managed by the ``for'' loop in \cref{alg:eps_regret}. 
Each phase keeps track of a solution $x_\ell$ and a cost threshold $\theta_\ell := \alpha^{\ell-1} C_{\max} = \beta^{(\ell-1)/2} C_{\max}$.
The invariant which, due to stochasticity,  we wish to maintain with high probability is that $\cost(x_\ell) \leq \theta_\ell$.
Thus, assuming $\cost(x_\ell) \leq \theta_\ell$, in phase $\ell$ we wish to either determine that $\cost(x_\ell) \leq \alpha \theta_{\ell} = \theta_{\ell+1}$, find some $x' \in \cN(x_\ell)$ with $\cost(x') \leq \alpha \cost(x_\ell) \leq \alpha \theta_\ell = \theta_{\ell+1}$, or determine that no such $x'$ exists.
In the first case we may set $x_{\ell+1} = x_{\ell}$ to maintain the invariant, while in the second we may set $x_{\ell+1} = x'$ to maintain the invariant.
In the last case, we will use the fact that $\cN$ satisfies \cref{def:improving_moves} and no $x'$ was found to conclude that $\cost(x_{\ell}) \leq \gamma \opt$, and thus we use the current solution $x_{\ell}$ for all remaining periods.

To achieve this, \cref{alg:eps_regret} first uses solution $x_\ell$ for $N_\ell$ time steps to get a rough estimate of its cost, then compares it to a threshold.
The threshold is set so that if $\cost(x_\ell) \leq \alpha^2 \theta_\ell$ then the estimated cost is highly likely to be smaller than the threshold.
As discussed above, we then move to phase $\ell+1$ with $x_{\ell+1} = x_\ell$.
If instead, the estimated cost is larger than the threshold, we run the subroutine   in \cref{alg:tester} which explores the neighborhood of $x_\ell$.
This subroutine returns a solution $x_{\new}$, which is either some $x' \in \cN(x)$ or $x_\ell$. 
If $x_{\new}$ is some $x' \in \cN(x)$, then we set $x_{\ell+1} = x_{\new}$ and move to phase $\ell+1$.
Critically, we will show that with high probability $\cost(x_{\new}) \leq \alpha \theta_\ell = \theta_{\ell+1}$ to maintain the invariant in this case.
Finally if $x_{\new} = x_\ell$, this indicates that nothing significantly better was found and so we break out of the ``for'' loop over phase $\ell$, setting $x_{\last} = x_\ell$ and using $x_{\last}$ until period $T$ since, by \cref{def:improving_moves}, $x_{\last}$ will satisfy $\cost(x_{\last}) \leq \gamma \opt$ with high probability.

As discussed, we want \cref{alg:tester} to either find a neighboring solution which improves upon $x_{\ell}$ by at least an $\alpha$-factor or indicate that no such improvement exists in $\cN(x_\ell)$, and for this to hold with high probability.
It is thus inevitable that a significant number of time steps are spent exploring the neighbourhood of $x_\ell$. Doing so without incurring high regret requires some care. 
A first attempt might involve sampling each neighboring solution $N_\ell$ times, as we did for $x_\ell$ in \cref{alg:eps_regret}, and comparing to an appropriate threshold.
This may potentially incur linear regret in later phases when a neighboring solution could be  significantly worse than the current solution $x_\ell$.
To avoid this, we  sample neighboring solutions cautiously over a series of ``sub-phases'' of increasing length (\cref{alg:tester}, line 4), where we progressively increase the amount of sampling only while the neighbouring solution has the potential to improve over $x_\ell$.
As a result, poor neighbours are eliminated in early subphases,  
reminiscent of the successive-elimination algorithm for stochastic multi-armed bandits~\citep{DBLP:conf/colt/Even-DarMM02}, and sampled geometrically fewer times than the best solutions (and $x_\ell$).

\begin{algorithm}[t]
    \caption{Local Search for Stochastic Combinatorial Bandits\label{alg:eps_regret}}
    \SetKwProg{proc}{Procedure}{:}{}
    \SetKwComment{Comment}{//}{}
    \SetKwFunction{local}{BanditLocalSearch}
    \SetKwFunction{test}{TestNeighborhood}
    \SetKw{Break}{break}
    \DontPrintSemicolon
    \KwData{Feasible set $\cX$, Max cost $C_{\max}$, $M = \sup_{x \in \cX} |\cN(x)|$, and Parameters $\beta$, $\gamma$ from \cref{def:improving_moves} }
    \KwResult{Sequence of solutions with $O(\log^3 T)$ $\gamma$-regret}
    \proc{\local{$\cX$, $C_{\max}$, $\beta$, $\gamma$}}{
        $x_1 \gets$ any solution in $\cX$ \;
        $\alpha \gets \sqrt{\beta}$ \;
        $\delta \gets (1- \alpha) / (1+\alpha)$ \;
        $\theta_1 \gets C_{\max}$\;
        \For{{\rm phase} $\ell = 1,2, \ldots$}{
            \Comment{Test the solution $x_\ell$ for the current phase $\ell$}
            $N_\ell \gets 3 C_{\max}\left( 4 \log T + \log M \right) / \left( \delta^2 \alpha^2 \theta_\ell \right)$\; 
            Use solution $x_\ell$ for $N_\ell$ periods\;
            $\widehat{\cost}(x_\ell) \gets $ average cost of using $x$ in these periods \;
            \eIf{$\widehat{\cost}(x_\ell) > \left( \frac{2\alpha^2}{1+\alpha}\right) \theta_\ell$}{
                $x_{\new} \gets $ \test{$x_\ell$, $\ell$, $\cN$} \;
                \eIf{$x_{\new} = x_\ell$}{
                    \Comment{$x_\ell$ is locally optimal with high probability}
                    $x_{\last} \gets x_\ell$ \;
                    \Break \;
                }{
                    \Comment{$x_{\new}$ is better than $x_\ell$ with high probability}
                    $x_{\ell+1} \gets x_{\new}$ \;
                }
            }{
            $x_{\ell+1} \gets x_\ell$\;
            }
            
            $\theta_{\ell+1} \gets \alpha \theta_\ell$
        }
        Use solution $x_{\last}$ until period $T$ \;
    
    }
\end{algorithm}

\begin{algorithm}[t]
    \caption{Neighboring Solution Tester \label{alg:tester}} 
    \SetKwProg{proc}{Procedure}{:}{}
    \SetKwComment{Comment}{//}{}
    \SetKwFunction{local}{LocalSearch}
    \SetKwFunction{test}{TestNeighborhood}
    \SetKw{Break}{break}
    \DontPrintSemicolon
    \KwData{Current solution $x_\ell$,  Phase number $\ell$, Neighborhood $\cN$}
    \KwResult{Some solution $x' \in \cN(x_\ell)$ or $x_\ell$}

    \proc{\test{$x_\ell$, $\ell$, $\cN$}}{
        \For{$x' \in \cN(x_\ell)$}{
            \Comment{Test the solution $x'$}
            ${\sf Better} \gets {\sf True}$ \;
            \For{{\rm sub-phase} $\ell' = 1,2, \ldots, \ell$}{ 
                $N_{\ell'} \gets 3 C_{\max}\left( 4 \log T + \log M \right) / \left( \delta^2 \alpha^2  \theta_{\ell'} \right)$\; 
                Use solution $x'$ for $N_{\ell'}$ periods\;
                $\widehat{\cost}(x') \gets $ average cost of using $x'$ in these periods \;
                \Comment{If $x'$ is not significantly better than $x$, move on from $x'$}
                \If{$\widehat{\cost}(x') > \left( \frac{2\alpha^2}{1+\alpha}\right) \theta_{\ell'}$}{
                    ${\sf Better} \gets {\sf False}$ \;
                    \Break \;
                }
            }
            \Comment{If $x'$ passes all tests, update $x_{\ell+1}$ to be $x'$}
            \If{${\sf Better} = {\sf True}$} {
                \Return $x'$ \;
            }
        }
        \Comment{Nothing significantly better found}
        \Return $x_\ell$\;
    }
\end{algorithm}

Let $\cA_{\sf Local}$ denote the algorithm described  by \cref{alg:eps_regret,alg:tester}.
Our main result is that  the regret of $\cA_{\sf Local}$ scales poly-logarithmically with $T$ as long as the underlying problem admits $(\beta, \gamma)$-improving moves.

\begin{theorem} \label{thm:gamma_regret_main_result}
    Suppose that $\cX$, $\cN$, and the cost function $\cost(\cdot) = \E_{z \sim \cD}[\cost(\cdot, z)]$ induced by $\cD$ admit $(\beta, \gamma)$-improving moves. 
    Then for all distributions $\cD$ and all sufficiently large $T$, we have 
    \[
    \regret_\gamma(\cA_{\sf Local}, \cD, T) = O\left( \frac{M C_{\max} \log^2 T}{\delta^2 \alpha^2 \log^2 \frac{1}{\alpha}} \left( \log T + \log M \right)\right),
    \]
    where $\cA_{\sf Local}$ is given by \cref{alg:eps_regret,alg:tester}, $\alpha = \sqrt{\beta}$, $C_{\max} = \max_{x \in \cX, z \in \cZ} \cost(x, z)$, and $\delta = \frac{1-\alpha}{1+\alpha}$.
\end{theorem}

\subsection{Analysis} 

We now perform the analysis which will yield \cref{thm:gamma_regret_main_result}.
First, we encapsulate our applications of concentration inequalities into the following lemma which we will apply several times.

\begin{lemma} \label{lem:concentration}
Fix a solution $x \in \cX$ and let $z_1, z_2, \ldots, z_N$ be  
$N := 3 C_{\max} (4 \log T + \log M ) / (\delta^2 \alpha^2 \theta )$ 
independent samples from $\cD$, where $\ell, M \in \N$, $\theta > 0$, $\alpha \in (0,1)$, and $\delta = \frac{1-\alpha}{1+\alpha} \in (0,1)$.
Let $\widehat{\cost}(x) = \frac{1}{N} \sum_{s=1}^N \cost(x,z_s)$. Then
\begin{enumerate}[label=(\alph*)]
    \item $\cost(x) \leq \alpha^2 \theta \implies \Pr_{z_1,z_2,\ldots,z_N \sim \cD^N}\left[ \widehat{\cost}(x) >  \frac{2\alpha^2}{1+\alpha} \theta \right] \leq 
    M^{-1}T^{-4}$, and
    \item $\cost(x) \geq \alpha \theta \implies \Pr_{z_1,z_2,\ldots,z_N \sim \cD^N}\left[ \widehat{\cost}(x) \leq  \frac{2\alpha^2}{1+\alpha} \theta \right] \leq 
    M^{-1}T^{-4}$.
\end{enumerate}
\end{lemma}
\begin{proof}
For brevity, we write $\Pr[\cdot] := \Pr_{z_1,z_2,\ldots,z_N \sim \cD^N}[\cdot]$.
Setting $X_s = \cost(x, z_s) / C_{\max}$ for all $s \in [N]$, we have that $X_s \in [0,1]$ and $\E[X_s] = \cost(x) / C_{\max}$.
Observe that $\frac{2\alpha^2}{1+\alpha} = (1+\delta) \alpha^2 = (1-\delta) \alpha$ for our choice of $\delta$.
For part (a), when  $\cost(x) \leq \alpha^2 \theta$,  observe that
\begin{align*}
     \Pr\left[ \widehat{\cost}(x) >  \frac{2\alpha^2}{1+\alpha} \theta \right] =  \Pr\left[ \frac{\widehat{\cost}(x)}{C_{\max}} >  (1+\delta) \frac{\alpha^2\theta}{C_{\max}} \right] = \Pr\left[ \frac{1}{N} \sum_{s=1}^N X_s > (1+\delta) \frac{\alpha^2\theta}{C_{\max}} \right].
\end{align*}
Since $X_s \in [0,1]$ and $\E[X_s] \leq \alpha^2 \theta / C_{\max}$, we may bound the right hand side using \cref{thm:gen_chernoff} as
\begin{align*}
     \Pr\left[ \frac{1}{N} \sum_{s=1}^N X_s > (1+\delta) \frac{\alpha^2\theta}{C_{\max}} \right] \leq \exp \left(-  \frac{\delta^2 N \alpha^2 \theta}{3 C_{\max}} \right) = 
     M^{-1}T^{-4},
\end{align*}
where the right hand side follows from our choice of $N$.

For part (b), when $\cost(x) \geq \alpha \theta$, it follows that
\begin{align*}
     \Pr\left[ \widehat{\cost}(x) \leq \frac{2\alpha^2}{1+\alpha} \theta \right] =  \Pr\left[ \frac{\widehat{\cost}(x)}{C_{\max}} \leq  (1-\delta) \frac{\alpha \theta}{C_{\max}} \right] = \Pr\left[ \frac{1}{N} \sum_{s=1}^N X_s \leq (1-\delta) \frac{\alpha\theta}{C_{\max}} \right].
\end{align*}
Again, since $X_s \in [0,1]$ and $\E[X_s] \geq \alpha \theta / C_{\max}$, we   bound the right hand side using \cref{thm:gen_chernoff} as
\begin{align*}
     \Pr\left[ \frac{1}{N} \sum_{s=1}^N X_s \leq (1-\delta) \frac{\alpha\theta}{C_{\max}} \right] \leq \exp \left(-  \frac{\delta^2 N \alpha \theta}{3 C_{\max}} \right) \leq  
     M^{-1}T^{-4},
\end{align*} 
where the right hand side follows from our choice of $N$ and  $\alpha \leq 1$. This completes the proof.
\end{proof}

\subsubsection{Handling Bad Events} \label{sec:bad_events_defs}

Let $L$ be the index of the (random) last phase that is encountered in \cref{alg:eps_regret}.
To bound the overall regret, we will bound the total regret in phases with index $\ell < L$ and also show that the regret in the last phase  $L$ is negligible.
Before we can do that, we need to define an appropriate sequence of ``bad events'' which are exceedingly unlikely, for which will be able to bound the regret when conditioning on their negation, as is standard.

\begin{definition} \label{def:bad_events}
For each phase $\ell$, let $\cM_\ell$ be the event that \cref{alg:eps_regret} makes it to phase $\ell$ and let $\cC_\ell$ be the event that $\cost(x_\ell) > \theta_\ell$.
Then we define the bad event in phase $\ell$ as  $\cB_\ell := \cM_\ell \wedge \cC_\ell$ and $\cG_\ell = \neg \cB_\ell$.
\end{definition}

We will show that these events are unlikely.
Intuitively, this will follow inductively by assuming that $\cB_{\ell}$ is unlikely, and then we can use the definition of our algorithm to bound $\Pr[\cB_{\ell+1} \mid \cG_{\ell}]$ and show that $\cB_{\ell+1}$ is unlikely via a standard decomposition.
The key step involves bounding $\Pr[\cB_{\ell+1} \mid \cG_{\ell}]$.  
\begin{lemma} \label{lem:bad_phase_prob_given_previous_good_phase}
    For each $\ell$, we have $\Pr[\cB_{\ell+1} \mid \cG_\ell] \leq T^{-3}$
\end{lemma}
Proving this lemma requires careful analysis of the \verb|TestNeighborhood| subroutine (\cref{alg:tester}), which we postpone momentarily.
For now, we show \Cref{lem:bad_phase_prob_given_previous_good_phase} implies a bound on $\Pr[\cB_\ell]$.

\begin{corollary} \label{cor:bad_probability_bound}
    For each $\ell$, we have $\Pr[\cB_\ell] \leq \ell T^{-3}$.
\end{corollary}
\begin{proof}
    We proceed by induction on $\ell$.
    For the base case when $\ell=1$, we have that $\cost(x_1) \leq C_{\max} \leq \theta_1$, and so $\Pr[\cB_1] = 0 \leq T^{-3}$, proving the base case.

    Inductively, we assume that $\Pr[\cB_{\ell}] \leq \ell T^{-3}$ and aim to bound $\Pr[\cB_{\ell+1}]$.
    By the law of total probability we have
    \begin{align*}
        \Pr[\cB_{\ell+1}] & = \Pr[\cB_{\ell+1} \mid \cB_{\ell}] \Pr[\cB_{\ell}] + \Pr[\cB_{\ell+1} \mid \cG_{\ell}] \Pr[\cG_{\ell}] \\
        & \leq \Pr[\cB_\ell] + \Pr[\cB_{\ell+1} \mid \cG_\ell] \\
        & \leq \ell T^{-3} + T^{-3} = (\ell+1)T^{-3}
    \end{align*}
    completing the proof of the inductive case and giving the corollary.
\end{proof}

\subsubsection{Analysis of the \texttt{TestNeighborhood} Subroutine} \label{sec:test_analysis}

We now turn to analyzing \cref{alg:tester}.
Consider a phase $\ell$ of \cref{alg:eps_regret} in which \cref{alg:tester} is called on solution $x_\ell$ and fix an iteration of the outer `for' loop in \cref{alg:tester} in which solution $x' \in \cN(x_\ell)$ is considered.
We call each iteration of the inner `for' loop in \cref{alg:tester} a \emph{sub-phase}, of which there are at most $\ell$ for each neighboring solution $x'\in \cN(x)$.
Ideally, if $\cost(x') > \alpha \theta_\ell$, then we do not make it past sub-phase $\ell$ when considering solution $x'$ (and actually, we need to be slightly more careful than this).
The tools we develop here are needed to bound the regret incurred by using each neighboring solution $x'$ across each sub-phase.
To start, we introduce the following definitions.

\begin{definition} \label{def:bad_neighbors}
    For each $x' \in \cN(x)$, we say that $x'$ is a bad neighbor of $x$ if $\cost(x') \geq \alpha \theta_\ell$.
    Additionally, let $\cN_B(x) \subseteq \cN(x)$ be the set of bad neighbors of $x$ and let $\ell'(x') \in [\ell]$ be the sub-phase index such that $\cost(x') \in [\alpha \theta_{\ell'(x')}, \theta_{\ell'(x')}]$ for each $x' \in \cN_B(x)$.
   
\end{definition}

\begin{definition} \label{def:bad_event_for_tester}
    Let $\cB'_\ell$ be the event that in phase $\ell$ some bad neighbor $x' \in \cN_B(x_\ell)$ makes it past sub-phase $\ell'(x')$ when \cref{alg:tester} is called.
\end{definition}

Note that if any bad neighbor $x'$ makes it past sub-phase $\ell$, then \cref{alg:tester} returns $x'$.  
Further, if any bad neighbor $x'$ makes it past sub-phase $\ell'(x')$, then we may incur too much regret from testing $x'$ in later sub-phases since $\cost(x') \geq \alpha \theta_{\ell'(x')}$.
We show that $\cB'_\ell$ is unlikely, and conclude that  bad neighbors are highly likely to be removed from contention early enough to avoid incurring high regret. 

\begin{lemma} \label{lem:bad_neigbor_elim}
     $\Pr[\cB'_\ell] \leq T^{-4}$.
\end{lemma}

\begin{proof}
    In the case that \cref{alg:tester} is not called in phase $\ell$, then the probability is at most 0 since $\cB'_\ell$ cannot happen.
    Formalizing this, let $\cT_\ell$ be the event that \cref{alg:tester} is called in phase $\ell$, then we have $\Pr[\cB'_\ell \mid \neg \cT_\ell] = 0$.
    In the other case where \cref{alg:tester} is called in phase $\ell$, then we need to show that it is unlikely for any bad neighbor $x'$ to make it past sub-phase $\ell'(x')$.
    To this end, fix a bad neighbor $x'$.
    If $x'$ doesn't make it to sub-phase $\ell'(x') \leq \ell$, then it clearly doesn't make it past sub-phase $\ell$.
    Thus in order for event $\cB'_\ell$ to happen due to $x'$, we must have that $x'$ makes it to sub-phase $\ell'(x')$ and makes it past this sub-phase.
    Now, $x'$ only makes it past sub-phase $\ell'(x')$ when $\widehat{\cost}(x') \leq (2\alpha^2/(1+\alpha)) \theta_{\ell'(x')}$, where $\widehat{\cost}(x')$ is the average of $N_{\ell'(x')}$ independent samples with distribution $\cost(x', Z)$, where $Z \sim \cD$.
    By part(b) of \cref{lem:concentration},   the probability of this is at most 
    $M^{-1}T^{-4}$. It follows that
    \begin{align*}
    \Pr[\cB'_\ell \mid \cT_\ell] & \leq \Pr\left[ \bigcup_{x' \in \cN_B(x)} \left(\widehat{\cost}(x') \leq \frac{2\alpha^2}{1+\alpha} \theta_{\ell'(x')} \right) \right] \\
    & \leq \sum_{x' \in \cN_B(x)} \Pr \left[\widehat{\cost}(x') \leq \frac{2\alpha^2}{1+\alpha} \theta_{\ell'(x')} \right] \leq T^{-4}
    \end{align*}
    Finally, observe that
    \[
    \Pr[\cB'_\ell] = \Pr[\cB'_\ell \mid \cT_\ell] \Pr[\cT_\ell] + \Pr[\cB'_\ell \mid \neg \cT_\ell]\Pr[\neg\cT_\ell] \leq T^{-4} (\Pr[\cT_\ell] + \Pr[\neg \cT_\ell]) = T^{-4}  
    \]
    which completes the proof.
\end{proof}

\Cref{lem:bad_neigbor_elim} establishes that \cref{alg:tester} filters out bad neighboring solutions with high probability.
In order to prove \cref{lem:bad_phase_prob_given_previous_good_phase}, we also need to show that good neighboring solutions are unlikely to be filtered out.
More specifically, there are three situations to consider (all under the condition that $\cost(x_\ell) \leq \theta_\ell$): (1) every solution in $\cN(x_\ell)$ is a bad neighbor, (2) some solution $x'\in\cN(x_\ell)$ satisfies $\cost(x') \leq \alpha^2 \theta_\ell$, and (3) not all neighboring solutions are bad, but all solutions that are not bad have $\cost(x') \in [\alpha^2 \theta_\ell, \alpha \theta_\ell]$.

In  situation (1), it is a simple corollary of \cref{lem:bad_neigbor_elim} that \cref{alg:tester} will output $x_\ell$ with high probability, indicating that we did not find an improved solution (and therefore $\cost(x_\ell) \leq \gamma \opt$).
In situation (2),  where there is a neighboring solution $x'$ with $\cost(x') \leq \alpha^2 \theta_\ell$, it is possible that \cref{alg:tester} considers and returns a different neighboring solution $x''$ with $\cost(x'') \in [\alpha^2 \theta_\ell, \alpha \theta_\ell]$ (again it is unlikely to return a bad neighbor due to \cref{lem:bad_neigbor_elim}). However, if $x'$ is ever considered  it will survive past sub-phase $\ell$ with high probability (\Cref{lem:concentration,lem:bad_neigbor_elim}) and   thus be returned by \cref{alg:tester}.
Either outcome is acceptable and guarantees  at least an $\alpha$-factor improvement.
In situation (3),  \cref{alg:tester} may output $x_\ell$, indicating no improved solution.   
This follows since, under the event that $\cost(x_{\ell}) \leq \theta_\ell$, each neighboring solution has $\cost(x') \geq \alpha \theta_\ell \geq \beta \cost(x_\ell)$, which by the $(\beta, \gamma)$-improving moves condition implies that $\cost(x_\ell) \leq \gamma \opt$. 
It is also acceptable for \cref{alg:tester} to output a solution with cost in the interval $[\alpha^2 \theta_\ell, \alpha \theta_\ell]$, since that   guarantees an $\alpha$-factor improvement.
The following lemma formalizes this.

\begin{lemma} \label{lem:test_alg_guarantee}
    Suppose that \cref{alg:tester} is run in phase $\ell$ with solution $x_\ell$ and let $x_{\new}$ be the random solution which it outputs.
    Define the events $\cF_{\ell,1}, \cF_{\ell,2}, \cF_{\ell,3}$ as follows:
    \begin{itemize}
        \item $\cF_{\ell,1} = \{ \cN_B(x_\ell) = \cN(x_\ell) \}$ (all neighbors of $x_\ell$ are bad)
        \item $\cF_{\ell,2} = \{ \exists x' \in \cN(x_\ell), \cost(x') \leq \alpha^2 \theta_\ell \}$ (there is a neighbor with a $\beta=\alpha^2$-factor decrease)
        \item $\cF_{\ell,3} = \{ \forall x' \in \cN(x_\ell) \setminus \cN_B(x_\ell) \neq \emptyset, \cost(x') \in [\alpha^2 \theta_\ell, \alpha \theta_\ell] \}$ (no neighbor has an $\beta=\alpha^2$-factor decrease, but there is one with an $\alpha$-factor decrease).
    \end{itemize}
    Then  
    \begin{enumerate}[label=(\alph*)]         
        \item $\Pr[x_{\new} \neq x_\ell \mid  \wedge \cF_{\ell,1}] \leq T^{-4}$
        \item $\Pr[\cost(x_{\new}) > \alpha \theta_\ell \mid \cF_{\ell,2}] \leq (\ell +1)T^{-4}$
        \item $\Pr[(\cost(x_{\new}) > \alpha \theta_\ell) \wedge (x_{\new} \neq x_\ell) \mid \cF_{\ell,3}] \leq T^{-4}$
    \end{enumerate}
\end{lemma} 
\begin{proof} 
    Part (a):
    Conditioned on $\cF_{\ell,1}$, the only way that $x_{\new} \neq x_\ell$ is if a bad  neighboring solution makes it past sub-phase $\ell$. This implies that   event $\cB'_\ell$ occurs, and it follows by \cref{lem:bad_neigbor_elim} that
    \[
    \Pr[x_{\new} \neq x_\ell \mid  \cF_{\ell,1}] \leq \Pr[\cB'_\ell] \leq T^{-4}.
    \] 

    Part (b): There are two ways we can have $\cost(x_{\new}) > \alpha \theta_\ell$ under condition $\cF_{\ell,2}$.
    One way is that a bad neighbor has been output as $x_{\new}$, and thus $\cB'_\ell$ has occurred as discussed previously.
    The other way is if $x_{\new} = x_\ell$ and $\cost(x_\ell) > \alpha \theta_\ell$, in which case it must be that the solution $x'$ with $\cost(x_\ell) \leq \alpha^2 \theta_\ell$ was considered but did not make it past sub-phase $\ell$.
    Let $\cB'_\ell(x')$ denote this latter event.
    Thus by a union bound we have
    \[
    \Pr[\cost(x_{\new}) > \alpha \theta_\ell \mid  \cF_{\ell,2}] \leq \Pr[\cB'_\ell] + \Pr[\cB'_\ell(x')].
    \]
    Again, from \cref{lem:bad_neigbor_elim}, $\Pr[\cB'_\ell] \leq T^{-4}$. It remains to   bound $\Pr[\cB'_\ell(x')]$.
    This event happens if $x'$ fails at least one of the checks that occur when $x'$ is considered in the inner `for' loop of \cref{alg:tester}.
    By part (a) of  \cref{lem:concentration}  and a union bound, we have
    \[
    \Pr[\cB'_\ell(x')] \leq \ell M^{-1} T^{-4} \leq \ell T^{-4}.
    \]
    Combining these bounds completes part (b).

    Part (c): We can only have $\cost(x_\ell) > \alpha \theta_\ell$ and $x_{\new} \neq x_\ell$ if some bad neighbor makes it past sub-phase $\ell$, which can only occur if $\cB'_\ell$ has occurred.
    Thus we have 
    \[
    \Pr[(\cost(x_{\new}) > \alpha \theta_\ell) \wedge (x_{\new} \neq x_\ell) \mid \cF_{\ell,3}] \leq \Pr[\cB'_\ell] \leq T^{-4}
    \]
    by \cref{lem:bad_neigbor_elim}.
    This yields the claim for the last case and completes the proof of the lemma.
\end{proof}

As a direct corollary of \cref{lem:test_alg_guarantee}, we get that \cref{alg:tester} is unlikely to output a solution with cost more than $\alpha \theta_\ell$ when run in phase $\ell$.

\begin{corollary} \label{cor:test_alg_cor}
    Suppose that \cref{alg:tester} is run in phase $\ell$ with solution $x_\ell$ and let $x_{\new}$ be the random solution which it outputs.  Then we have that $\Pr[\cost(x_{\new}) > \alpha \theta_\ell] \leq (\ell+1) T^{-4}$.
\end{corollary}
 
\begin{proof}
    Under event $\cF_{\ell, 1}$, the only way for \cref{alg:tester} to output a bad solution is to output something other than $x_\ell$.
    Similarly, the only for \cref{alg:tester} to output a bad solution under event $\cF_{\ell, 3}$ is to output something other than $x_{\ell}$ which happens to be bad. 
    Then it follows from the law of total probability and \Cref{lem:test_alg_guarantee} that
    \begin{align*}
    \Pr  [\cost(x_{\new}) > \alpha \theta_\ell] = & \Pr[x_{\new} \neq x_\ell \mid \cF_{\ell,1}]\Pr[\cF_{\ell, 1}] \\
     & + \Pr[\cost(x_{\new}) > \alpha \theta_\ell \mid \cF_{\ell,2}] \Pr[\cF_{\ell,2}]  \\  &+  \Pr[(\cost(x_{\new}) > \alpha \theta_\ell) \wedge (x_{\new} \neq x_\ell) \mid \cF_{\ell,3}] \Pr[\cF_{\ell,3}] \\
     \leq & (\ell+1)T^{-4}(\Pr[\cF_{\ell, 1}] + \Pr[\cF_{\ell, 2}] + \Pr[\cF_{\ell, 3}]) \\
     = & (\ell+1)T^{-4}.
    \end{align*} 
\end{proof}

\subsubsection{Proof of \texorpdfstring{\cref{lem:bad_phase_prob_given_previous_good_phase}}{Bad Event Lemma}} \label{sec:proof_of_bad_phase_given_previous_good_phase} 

Now we return to the proof of \cref{lem:bad_phase_prob_given_previous_good_phase} which will utilize the tools we developed from the analysis of \cref{alg:tester}.

\begin{proof}[Proof of \cref{lem:bad_phase_prob_given_previous_good_phase}]
    Recall that $\cG_\ell = \neg  \cC_\ell \vee \neg \cM_\ell$, i.e., the good event for phase $\ell$ happens if either \cref{alg:eps_regret} doesn't make it to phase $\ell$ or we have $\cost(x_\ell) \leq \theta_\ell$. 
    First, we claim that $\Pr[\cB_{\ell+1} \mid \cG_\ell] \leq \Pr[\cB_{\ell+1} \mid \neg \cC_\ell]$.
    To see this, we have:
    \begin{align*}
        \Pr[\cB_{\ell+1} \mid \cG_\ell] & = \frac{\Pr[\cB_{\ell+1} \wedge (\neg \cC_\ell \vee \neg \cM_\ell)]}{ \Pr[\neg \cC_\ell \vee \neg \cM_\ell ]} \\
        & = \frac{\Pr[(\cB_{\ell+1} \wedge \neg \cC_\ell) \vee (\cB_{\ell+1} \wedge \neg \cM_{\ell})]}{\Pr[\neg \cC_\ell \vee \neg \cM_\ell ]} \\
        & \leq \frac{\Pr[\cB_{\ell+1} \wedge \neg \cC_\ell]}{{\Pr[\neg \cC_\ell \vee \neg \cM_\ell ]}} + \frac{\Pr[\cB_{\ell+1} \wedge \neg \cM_{\ell}]}{\Pr[\neg \cC_\ell \vee \neg \cM_\ell ]} \\
        & \leq \frac{\Pr[\cB_{\ell+1} \wedge \neg \cC_\ell]}{{\Pr[\neg \cC_\ell  ]}} + \frac{\Pr[\cB_{\ell+1} \wedge \neg \cM_{\ell}]}{\Pr[ \neg \cM_\ell ]} \\
        & = \Pr[\cB_{\ell+1} \mid \neg \cC_\ell] + \Pr[\cB_{\ell+1} \mid \neg \cM_{\ell}]\\
        & = \Pr[\cB_{\ell+1} \mid \neg \cC_\ell]
    \end{align*}
    The first line follows from the definition of $\cG_\ell$ and conditional probability while the second line follows from the distributive law for $\wedge$ and $\vee$.
    We use a union bound in the third line and the observation that $\cG_\ell$ contains both $\neg \cC_\ell$ and $\neg \cM_\ell$ in the fourth line.
    To finish, the fifth line is again the definition of conditional probability and the last step uses the observation that $\Pr[\cB_{\ell+1} \mid \neg \cM_\ell] = 0$ since it is impossible for the algorithm to make it to phase $\ell+1$ if it has not made it to phase $\ell$.

    Thus we may focus on bounding $\Pr[ \cB_{\ell+1} \mid \neg \cC_\ell]$.
    Under   condition $\neg \cC_\ell$,   $\cost(x_\ell) \leq \theta_\ell$ by definition.
    We analyze three cases depending on how $\cost(x_\ell)$ relates to $\theta_\ell$.
    \begin{description}
        \item[Case 1: $\cost(x_\ell) \leq \alpha^2 \theta_\ell$.]
        In this case, we have $\cost(x_\ell) \leq \alpha \theta_{\ell+1}$ and the current cost is small enough to move to phase $\ell+1$.
        Denote the event that $\cost(x_\ell) \leq \alpha^2 \theta_{\ell}$ by $\cE_{\ell, 1}$.
        We will show that $\Pr[\cB_{\ell+1} \mid \neg \cC_\ell \wedge \cE_{\ell, 1}]$ is small.
        If \cref{alg:eps_regret} moves to phase $\ell+1$ with $x_{\ell+1} = x_\ell$, then $\cost(x_{\ell+1}) < \theta_{\ell+1}$ so $\cB_{\ell+1}$ doesn't happen.
        Thus the only way for $\cB_{\ell+1}$ to happen is if the ``if'' statement on line 10 evaluates to true, which occurs when $\widehat{\cost}(x_\ell) > 2\alpha^2\theta_\ell/(1+\alpha)$.
        We will show that this happens with small probability.
        Since $\widehat{\cost}(x_\ell)$ is the average of $N_\ell$ independent samples each with mean $\cost(x_\ell)$, by the first part of \cref{lem:concentration} and the discussion above, we have
        \begin{align} \label{eq:bad_prob_lemma_case_one}
        \Pr[\cB_{\ell+1} \mid \neg \cC_{\ell} \wedge \cE_{\ell, 1}] \leq \Pr\left[ \widehat{\cost}(x_\ell) > \frac{2\alpha^2}{1+\alpha} \theta_\ell \right] \leq M^{-1}T^{-4} \leq T^{-4},
        \end{align}
        completing the argument for this case.
    
        \item[Case 2: $\alpha \theta_\ell \leq \cost(x_\ell) \leq \theta_\ell$.]  
        For this case, we want \cref{alg:tester} to be called with high probability and that it succeeds with high probability.
        Denote the event that $\alpha \theta_\ell \leq \cost(x_\ell) \leq \theta_\ell$ by $\cE_{\ell, 2}$.
        In this case, there are two ways that we could fail to satisfy $\cost(x_{\ell+1}) \leq \theta_{\ell+1}$. 
        The first  is by skipping the step which explores the neighborhood and moving to the next phase with $x_{\ell+1} = x_\ell$.
        This occurs when the ``if'' statement in line 10 of \cref{alg:eps_regret} evaluates to $\false$, which happens with probability at most $M^{-1}T^{-4}$ by part (b) of  \cref{lem:concentration}.
        Assuming   \cref{alg:tester} is called, the second way that we could fail to satisfy $\cost(x_{\ell+1}) \leq \theta_{\ell+1}$ is if \cref{alg:tester} returns a bad neighboring solution.
        \cref{cor:test_alg_cor} shows that this occurs with probability at most $(\ell+1)T^{-4} \leq T^{-3}$.
        Combining the bounds completes the analysis of  this case. 
        
        \item[Case 3: $\alpha^2 \theta_\ell < \cost(x_\ell) < \alpha \theta_\ell$.]
        This last case is more flexible; we use $\cE_{\ell,3}$ to denote the event that $\cost(x_\ell) \in (\alpha^2 \theta_\ell, \alpha \theta_\ell)$.
        Since $\cost(x_\ell) \leq \alpha \theta_\ell$, we will be safe in either the situation that $\widehat{\cost}(x_\ell) > 2\alpha^2 \theta_\ell / (1+\alpha)$ (and so we set $x_{\ell+1} = x_\ell$) or not (and so we use \cref{alg:tester} to determine $x_{\ell+1}$).
        In the first situation, we clearly have $\cost(x_\ell) \leq \alpha \theta_\ell = \theta_{\ell+1}$.
        In the other situation, we use the same argument as in the previous case to argue that the probability of setting $x_{\ell+1}$ to a bad neighbor is small.
        Thus we conclude that $\Pr[\cB_{\ell+1} \mid \neg \cC_{\ell} \wedge \cE_{\ell, 3}] \leq T^{-4}$
    \end{description}
    Since the three cases above are exhaustive, i.e., $\cE_{\ell, 1} \vee \cE_{\ell, 2} \vee \cE_{\ell, 1} \equiv \true$ when conditioned on $\cC_{\ell}$, we conclude by the law of total probability that
    \[
    \Pr[\cB_{\ell+1} \mid \cC_\ell] = \sum_{i=1}^3 \Pr[\cB_{\ell+1} \mid \cC_\ell \wedge \cE_{\ell,i}] \Pr[ \cE_{\ell,i} \mid \cC_\ell]\leq T^{-3}  \sum_{i=1}^3 \Pr[ \cE_{\ell,i} \mid \cC_\ell] = T^{-3},
    \]
    completing the proof of this lemma.
\end{proof}

\subsubsection{Bounding the Number of Phases} \label{sec:number_of_phases}

Let $L$ be the index of the last complete phase, i.e., \cref{alg:eps_regret} either finds a local solution $x_{\last}$ in phase $L$ (which is then played until period $T$) or period $T$ is reached in phase $L+1$.
We will show that $L$ is logarithmic in $T$, which will be helpful to establish our regret bound later.

\begin{lemma} \label{lem:phase_bound}
    $L \leq \log(T) / \log ( \frac{1}{\alpha})$
\end{lemma}
\begin{proof}  
    By definition of $L$ we know that $\sum_{\ell=1}^L N_\ell \leq T$, since this is just counting the number of plays due to line 7 in \cref{alg:eps_regret}.  
    From the definition of $N_\ell$, we have:
    \begin{align*}
        T &\geq \sum_{\ell=1}^L N_\ell \\
        & = \sum_{\ell=1}^L \left( \frac{3 C_{\max}}{\delta^2 \alpha^2 \theta_\ell} \left(  4 \log  T + \log M \right)  \right) \\
        & \geq \left( \frac{12 \log T}{\delta^2\alpha^2} \right) \sum_{\ell=1}^L \left(\frac{1}{\alpha} \right)^{\ell-1} \\
        & = \left( \frac{12 \log T}{\delta^2\alpha^2} \right) \left(  \frac{\left(\frac{1}{\alpha} \right)^{L} - 1}{\frac{1}{\alpha} -1} \right).
    \end{align*}
    Rearranging this inequality to solve for $L$ and simplifying yields result.
\end{proof} 

\subsubsection{Final Regret Analysis} \label{sec:final_regret_bound}  

To complete the analysis, we separately analyze the $\gamma$-regret incurred by the algorithm   during phases 1 to $L$ and the $\gamma$-regret after phase $L$.

\paragraph{Regret from Phases $1$ to $L$:} 
Consider phase $\ell$.
We incur regret from playing the solution $x_\ell$ and also from playing   neighboring solutions when  \cref{alg:tester} is called.
Let $R_{\ell,1}$ and $R_{\ell, 2}$ represent each of these quantities, respectively.
We start by bounding $R_{\ell, 1}$.
Since $\opt \geq 0$, by the law of total expectation we have:
\begin{align*}
    \E[R_{\ell, 1}] & = \E[R_{\ell, 1} \mid \neg \cB_{\ell}] \Pr[ \neg \cB_\ell] + \E[R_{\ell, 1} \mid  \cB_{\ell}] \Pr[ \cB_\ell] \\
    & \leq \theta_\ell N_\ell + C_{\max} N_{\ell} \ell T^{-3} \\
    & \leq      \frac{12 C_{\max}}{\delta^2 \alpha^2}\left( \log T + \log M \right) + C_{\max} T^{-1}.
\end{align*}
The first inequality follows since $\cost(x_\ell) \leq \theta_{\ell}$ under $\neg \cB_{\ell}$ (giving the first term) and \cref{cor:bad_probability_bound} (which bounds $\Pr[\cB_\ell])$.
The second inequality follows from the definition of $N_\ell$ and the fact that $N_\ell$ and $\ell$ are both smaller than $T$.

Bounding $R_{\ell, 2}$ is similar, but we have to account for each sub-phase.
We see that:
\begin{align*}
    \E[R_{\ell, 2}] & = \E[R_{\ell, 2} \mid \neg \cB'_{\ell}] \Pr[ \neg \cB'_\ell] + \E[R_{\ell, 2} \mid  \cB'_{\ell}] \Pr[ \cB'_\ell] \\
    & \leq M\sum_{\ell'=1}^\ell \theta_{\ell'} N_{\ell'} + C_{\max}\cdot T \cdot  T^{-4} \\
    & \leq  \frac{12 M \ell C_{\max} }{\delta^2 \alpha^2} \left( \log T + \log M \right) + C_{\max} T^{-3}.
\end{align*}
The first term in the first inequality follows since in the worst case we play each neighboring solution in all sub-phases, but under event $\neg\cB'_\ell$ we never play a solution with cost more than $\theta_{\ell'}$ in each sub-phase $\ell'$.
The second term follows since, under $\cB'_\ell$, we can't play a solution of cost more than $C_{\max}$ for more than $T$ steps and $\Pr[\cB'_\ell] \leq T^{-4}$ by \cref{lem:bad_neigbor_elim}.

\paragraph{Regret Incurred after Phase $L$:}
Let $R_{>L}$ denote the regret incurred after phase $L$.
We need to consider both the case where \cref{alg:eps_regret} terminated with a solution $x_{\last}$ and when it did not.
Denote the former event by $\cL$.
Then we have
\begin{align*}
    \E[R_{>L}] =\E[R_{>L} \mid \cL]\Pr[\cL] + \E[R_{>L} \mid \neg \cL]\Pr[\neg\cL].
\end{align*}
We now focus on each conditional expectation and show that the regret is small in each.

Consider the first case where we condition on $\cL$, and so we have that \cref{alg:eps_regret} terminated with some solution $x_{\last}$ which is used for all remaining rounds.
This happens when \cref{alg:tester} is run in phase $L$ and returns the solution $x_L$.
We want to show that it is unlikely for $x_L$ to satisfy $\cost(x_L) > \gamma \opt$.
Suppose  $\cost(x_L) > \gamma \opt$, then by \cref{def:improving_moves}, we have that $\cN(x_{\ell})$ contains a solution $x'$ with $\cost(x') \leq \beta \cost(x_L) = \alpha^2 \cost(x_L)$.
This means that event $\cF_{\ell, 2}$ has occurred, and by \cref{lem:test_alg_guarantee} the probability of \cref{alg:tester} yielding $x_L$ as output is at most $(\ell+1)T^{-4} \leq T^{-3}$.
If, instead,  $\cost(x_L) \leq \gamma \opt$, we incur zero $\gamma$-regret.
Together,  $\E[R_{>L} \mid \cL] \leq C_{\max} T \cdot T^{-3}$.

Now suppose $\cL$ does not occur.
In this case,   period $T$ was reached before finding an approximate local optimum, midway through executing an incomplete phase $L+1$.
We can  use similar techniques to bound $R_{>L}$ for this incomplete phase as we for     $R_\ell$ above, in fact, we can simply consider $R_{>L}$ to be  $R_{L+1}$. 
It follows from the same arguments as before that\[
\E[R_{>L} \mid \neg\cL] = \E[R_{L+1}] \leq \frac{12 (M+1) (L+1) C_{\max} }{\delta^2 \alpha^2} \left( \log T + \log M \right) + 2C_{\max} T^{-1}.
\]
Putting everything together, 
\[
\E[R_{>L}] \leq \frac{12 (M+1) (L+1) C_{\max} }{\delta^2 \alpha^2} \left( \log T + \log M \right) + 3C_{\max} T^{-1}.
\]

\paragraph{Combining the Bounds:}
Summing all the terms together, we have:
\begin{align*}
\regret_{\gamma}(\cA, \cD, T) &\leq \sum_{\ell=1}^L \left( \E[R_{\ell, 1}] +\E[R_{\ell, 2}] \right) + \E[R_{>L}] \\
& \leq O \left(  \frac{M C_{\max}}{\delta^2 \alpha^2} \left( \log T + \log M \right) \right) \sum_{\ell=1}^L \ell + O \left( \frac{M C_{\max} L}{\delta^2 \alpha^2} \left( \log T + \log M \right) \right) \\
& \leq O \left( \frac{M C_{\max} L^2}{\delta^2 \alpha^2} \left( \log T + \log M \right)\right) \\
& \leq O\left( \frac{M C_{\max} \log^2 T}{\delta^2 \alpha^2 \log^2 \frac{1}{\alpha}} \left( \log T + \log M \right)\right),
\end{align*}
where in the last step we used the bound on $L$ from \cref{lem:phase_bound}.
This establishes \cref{thm:gamma_regret_main_result}.

\section{Applications of Our Framework} \label{sec:applications}

\subsection{Stochastic Completion Time Scheduling on a Single Machine}\label{ssec:completiontime1}
We consider the problem of scheduling stochastic jobs to minimize total completion time on a single machine.  
Let $\cD$ be a distribution over job sizes which is  unknown to the scheduler.
There are $n$ jobs with stochastic sizes $P = (P_1, P_2, \ldots, P_n) \sim \cD$. 
Let $\mu_j := \E[P_j]$   denote the expected size of job $j$. We only require independence across time periods; a pair of jobs $j$ and $j'$ may have correlated size distributions.
We also make a standard assumption that the job size distribution is bounded and normalized to be in $[0,1]$. 

A schedule is given by a permutation $\pi:[n] \to [n]$ specifying that job $j$ is scheduled in position $\pi(j)$. Let $\Pi_n$ be the permutations on $[n]$, the feasible schedules. 
Given a schedule $\pi$,   let $C_j(\pi, P):= P_j + \sum_{j' : \pi(j') < \pi(j)} P_{j'}$ be the completion time of job $j$ under schedule $\pi$ and sizes $P$.
After reordering the summation, the cost of a schedule $\pi$ is
\begin{align} \label{eqn:weighted_completion_time}
    \cost(\pi, P) = \sum_{j=1}^n  C_j(\pi, P)  = \sum_{j=1}^n (n - \pi(j) + 1) P_j.
\end{align}
 
Let $\opt := \min_\pi \E_{P \sim \cD}[\cost(\pi, P )]$. 
We focus on learning from limited feedback using  fixed, non-preemptive schedules.
In the single machine case full-information case,   it is well known that the optimal schedule is given by applying Smith's rule to the expected job sizes: $\pi^*$ orders the jobs in non-decreasing order of $\mu_j$, or shortest expected processing time first~\citep{Rothkopf1966,Smith1956}.   

When learning to schedule with bandit feedback,   
the scheduler gains information by interacting  with unknown $\cD$ through the following process which occurs in each of  $T$ discrete time periods.
In period $t \in [T]$, the scheduler commits to a schedule $\pi^t$ before job sizes are realized according to   an independent sample $P^t \sim \cD$. Then the scheduler observes $\cost(\pi^t, P^t)$,  the realized total completion time of schedule $\pi^t$,  but does not observe the realized job sizes $P^t$.  

For a schedule $\pi$, define $\err(\pi):= \E[\cost(\pi, P)] - \opt$ to measure the expected regret we incur by using $\pi$ for one step. 
Recall that in this case an optimal schedule $\pi^*$ must schedule $j$ after $j'$ if and only if $\mu_j > \mu_{j'}$. 
We say that $\pi$ \emph{inverts} a pair of jobs $j, j'$ with $\mu_j > \mu_{j'}$ if instead it has  $\pi(j) < \pi(j')$. The following proposition shows that we can write $\err(\pi)$ in terms of inversions in $\pi$ from an optimal schedule.

\begin{proposition}[\citet{DBLP:conf/spaa/LindermayrM22}] \label{prop:error_decomposition}
    For any schedule $\pi$,  
    \[
    \err(\pi) = \sum_{j,j' : \mu_j > \mu_{j'}} \left( \mu_j - \mu_{j'} \right) \bfo\{ \pi(j) < \pi(j') \}.
    \]
\end{proposition}
Let $\cN(\pi) = \{ \pi'\in \Pi_n \mid \exists i\neq j, \pi(i) = \pi'(j), \pi(j) = \pi'(i), \text{ and }\pi(k) = \pi'(k) \text{ for all } k\not\in [n]\setminus\{i,j\}\}$ be the the neighborhood of schedules that consists of swapping a single pair of jobs. We show this admits $(\beta, \gamma)$-improving moves, with $\gamma = 1 + \epsilon$ for any $\epsilon > 0$ and $\beta = 1 - \epsilon/n^2.$

\begin{lemma} \label{lem:swap_lemma}
    Let $\pi$ be any schedule with $\cost(\pi) > (1+\epsilon)\opt$ with $\epsilon \in (0,1)$, then there exists a schedule $\pi'$ obtainable from $\pi$ by swapping a single pair of jobs such that
    \[
    \cost(\pi') \leq \left( 1 - \frac{\epsilon}{n^2} \right)\cost(\pi).
    \]
\end{lemma}

\begin{proof}
First we lower bound $\cost(\pi) - \opt$ by $\epsilon  \cdot\cost(\pi)/ 2$ since
\begin{align*}
    \cost(\pi) - \opt & \geq \cost(\pi) - \frac{\cost(\pi)}{1+\epsilon}  = \frac{\epsilon \cdot\cost(\pi)}{1+\epsilon} \geq \frac{\epsilon \cdot \cost(\pi)}{2}.
\end{align*}
Next, we upper bound $\cost(\pi) - \opt = \err(\pi)$ as
\begin{align*}
\err(\pi) = \sum_{j, j': \mu_j > \mu_{j'}} (\mu_j - \mu_{j'}) \bone\{\pi(j) < \pi(j')\} \leq \binom{n}{2} \cdot \max_{j, j': \mu_j > \mu_{j'}} (\mu_j - \mu_{j'}) \bone\{\pi(j) < \pi(j')\}
\end{align*}
where we use the fact that there are at most $\binom{n}{2}$ terms in the sum.
Combining these bounds, we conclude that there exists a pair $j,j'$ with $\mu_j > \mu_j'$ and $\pi(j) < \pi(j')$ such that 
\[
\mu_j - \mu_{j'} \geq \frac{\epsilon \cdot \cost(\pi)}{2 \binom{n}{2}} \geq \frac{\epsilon \cdot \cost(\pi)} {n^2}.\]
For the  new schedule $\pi'$ created by swapping the positions of $j$ and $j'$ in $\pi$, we observe that
\[
\cost(\pi) - \cost(\pi') = (\pi(j') - \pi(j))(\mu_j - \mu_{j'}) \geq(\mu_j - \mu_{j'}) \geq \frac{\epsilon  \cdot\cost(\pi)}{n^2},\]
which uses the fact that  $\pi(j') - \pi(j) \geq 1$ by construction.  
The lemma follows after rearranging.
\end{proof}

\begin{corollary} \label{corr_scheduling}
    Consider the problem of scheduling to minimize total completion time with stochastic job sizes drawn from distribution $\cD$. For any $\epsilon > 0$  there is an algorithm $\cA$ (\cref{alg:eps_regret}) with
    \[
    \regret_{1+\epsilon}(\cA_, \cD, T)  
    = O\left(\frac{n^{12} \log^3 T}{\epsilon^4  }  \right).
    \] 
\end{corollary}
\begin{proof}
    We apply \Cref{thm:gamma_regret_main_result} and simplify with problem-specific parameters. For the swap neighborhood, the maximum neighborhood size is $M\leq n^2$. Job sizes are at most 1, so $C_{max} \leq n^2.$ We can bound 
    \[
    \log(1/\alpha) = -\frac{1}{2}\log(1 - \epsilon/n^2) \geq \frac{1}{2}\frac{\epsilon}{n^2}
    \] since $\log(1 - x) \leq -x$ for $x> -1$. We can also bound \[
    \delta^2 \beta = \frac{(1 - \epsilon/n^2)(\epsilon/n^2)^2}{(1 + \sqrt{1 - \epsilon/n^2})^4} \geq 16 (1 - \epsilon/n^2)(\epsilon/n^2)^2 \geq 8\epsilon^2/n^4
    \] since $1 - \epsilon/n^2 < 1,$ implying $(1 + \sqrt{1 - \epsilon/n^2})^4 \leq 16$. The result follows after substitution and observing $\log n \leq \log T$.
\end{proof}
 
\subsection{Minimum Cost Base of a Matroid}\label{mst} 
We now apply our framework to the online stochastic bandit setting for finding a minimum cost base in a matroid. 
Let $M = (\cS, \I)$ be a matroid on ground set $\cS$ of cardinality $|\cS|=n$ with family of independent sets $\I$. Let $\B$ be its family of bases, or maximal independent sets, so that  each that $B\in \B$ satisfies $B\in \I$ and there is no $s\in S\setminus B$ such $B\cup s\in \I$. The \emph{rank} of set $S$, denoted $r(S)$, is the cardinality of the largest independent subset of $S$. The rank of matroid $M$ is the cardinality of a base, so $r(M) = |B|.$

In the minimum cost base problem, each element $s\in \cS$ is associated a cost and the objective is to find a base of minimum total cost. 
In the online stochastic bandit setting, the learner acts over a sequence of $T$ time periods, selecting base $B^t$ in period $t=1, \ldots, T$. 
In each period $t$, stochastic costs $Z^t =   \{Z_s^t\}_{s\in\cS}\sim \cD$ realize for each element $s \in \cS$, where $\cD$ is a joint distribution over edge costs. 
In time step $t$, given underlying realizations $z^t$, the learner observes only the  cost of base $B^t$, which  is $\cost(B^t, z^t) = \sum_{s\in B^t} z_s^t$. 
Let $\mu_e = \E_{Z \sim \cD}[Z_e]$ denote the expected cost of edge $e$, and let $B^* = \arg\min_{B'\in \B} \E_{Z\sim \cD}[\cost(B', Z)]$ be the base with lowest expected cost. 
 
We now define the local search neighborhood which will allow us achieve low regret online. A \emph{circuit} $C$ is a minimally dependent set: $C\not\in \cI$ but every subset of $C$ is independent.
\begin{definition} \label{def:mst_fundamental_cycle}
Given a base B of matroid $M$ and $s\in \cS\setminus B$, there is a unique circuit $C(s, B)$ such that $x\in C(s, B) \subseteq B\cup s$. Furthermore, for each $t\in C(s, B) $, $(B\cup s) \setminus t$ is also a base.  
\end{definition}

\begin{definition}
Given matroid $M$ with bases $\B$, let $\cN: \B \to 2^\B$ be the circuit swap neighborhood map. More formally, for each $B\in \B$, let 
\[
\cN(B) = \{B'\in \B \mid B' = (B\cup s)\setminus t \text{ for some } s\in \cS\setminus B, t \in C(s,B)   \},
\]
so $\cN(B)$ is the set of bases reachable from $B$ by adding an element $s\not\in B$ to B,  followed by removing an element $t$ from the resulting circuit $C(s, B).$
\end{definition}

We will make use of the following well-known property of matroids. 
\begin{lemma}[\citet{frank2011connections}]  \label{lem:matroid-bijection} If $B_1, B_2 \in \B$, then there exists a bijection $f:(B_1 \setminus B_2)\rightarrow (B_2 \setminus B_1)$ such that $B_1 \setminus \{x\} \cup f(x) \in \B$ for every $x\in B_1 \setminus B_2$. 
\end{lemma}

 We can now show that the circuit swap neighborhood admits $(\beta, \gamma)$-improving moves.

\begin{lemma} \label{lem:mst_improving_swaps}
    Given matroid $M$, the circuit swap neighborhood admits $(\beta, \gamma)$-improving moves with $\gamma = 1 + \epsilon$ for any $\epsilon > 0$ and $\beta = 1-\epsilon/(2r(M))$.
\end{lemma}
\begin{proof}
Consider $B\in \B$ satisfying $\cost(B) > (1 + \epsilon) \cost(B^*)$  for $\epsilon \in (0,1)$. 
It follows that ,
    \[
    \Delta = \cost(B) - \cost(B^*) \geq \cost(B) - \frac{\cost(B)}{1+\epsilon}  = \frac{\epsilon \cdot\cost(B)}{1+\epsilon} \geq \frac{\epsilon \cdot \cost(B)}{2}.
    \]

    We want to show there exists $B' \in \cN(B)$ satisfying $\cost(B') \leq \beta \cost(B)$.
    It is easy to see that
    \[
    \Delta = \cost(B) - \cost(B^*) = \sum_{e \in B \setminus B^*} \mu_e - \sum_{e \in B^* \setminus B} \mu_e. 
    \]
    From Lemma~\ref{lem:matroid-bijection},  construct a bijection $f:(B \setminus B^*)\rightarrow (B^* \setminus B)$ such that $B \setminus \{s\} \cup f(s)$ is a base for every $s\in B \setminus B^*$. If we start from $B$, then delete $s$ and add $f(s)$ from $B$  for all $s \in B \setminus B^*$, we transform from base $B$ to $B^*$. Since there are at most $|B| = r(M)$ items in $B \setminus B^*$, there must exist  $s'\in B \setminus B^*$, for which $B \setminus s'\cup f(s')$ improves the cost by at least $\Delta/\rho(M)$. Let $T' =  B \setminus s' \cup f(s') \in \cN(B)$. Then 
     \[
     \cost(B) - \cost(B') \ge \frac{\Delta}{\rho(M)}  \geq \frac{\epsilon \cdot \cost(B)}{2 r(M) }.
     \]
  
\end{proof}

\begin{corollary}
    Consider the problem of finding a minimum cost base in a matroid $M = (\cS, \I)$ with $|\cS| = n$ and rank $r(M)=r$ when element costs are  stochastically drawn from $\cD$. For any $\epsilon > 0$, there is an algorithm $\cA$ (\cref{alg:eps_regret}) with
    \[
    \regret_{1 + \epsilon}(\cA, \cD, T) = O\left( \frac{ n r^6 \log^2 T}{\epsilon^4} \left(\log T + \log n \right) \right)
    = O\left( \frac{ n r^6 \log^3 T}{\epsilon^4}  \right).
    \]    
\end{corollary}
\begin{proof}
    We apply \cref{lem:mst_improving_swaps} and \cref{thm:gamma_regret_main_result} with $(\beta, \gamma) = (1-\epsilon/2r, 1+\epsilon)$ for $\epsilon > 0$.
    Thus we just need to bound $M$, $C_{\max}$, $\frac{1}{\alpha^2}$, $\frac{1}{\delta^2}$, and $ \log \frac{1}{\alpha}$.
    It is easy to see that $M \leq nr$   and that $C_{\max} \leq r$.
    Next, we can see that $\frac{1}{\alpha^2} = O(1)$ for $\epsilon \leq r$ since $\alpha^2 = 1- \epsilon/ 2r \geq 1/2$.
    For $\frac{1}{\delta^2}$, we have:
    \[
    \delta = \frac{1-\alpha}{1+\alpha} = \frac{1-\alpha^2}{(1+\alpha)^2} = \frac{\epsilon / 2r}{1+2 \alpha + \alpha^2} \geq \frac{\epsilon}{8r}.
    \]
    Finally, for $\log(1/\alpha)$ we have:
    \[
    \log \frac{1}{\alpha} = \log \left(\frac{1}{1-\epsilon/2r} \right)^{1/2} = -\frac{1}{2} \log \left( 1-\epsilon/2r \right) \geq -\frac{1}{2} \log \exp \left( - \epsilon / 2r\right) = \frac{\epsilon}{4r}.
    \]
    Combining in \cref{thm:gamma_regret_main_result} yields the bound.
\end{proof}

\subsection{Uncertain \texorpdfstring{$k$}{K}-Median Clustering}

Our final application concerns $k$-median clustering where the set of points to be clustered are drawn at random from an unknown distribution.
The following setup was studied by \citet{cormode2008approximation} and \citet{DBLP:conf/pods/GuhaM09} for a several of $k$-clustering objectives in the offline setting.
An instance consists of $n$ points arriving from a metric space $(\cM, d)$ of diameter 1.
The location of the $i$'th point is uncertain and the objective is to choose $k$ cluster centers which minimize the total expected distance between each arriving point and its closest cluster center.
We let $Z_i \in \cM$ denote the random realization of the $i$'th point.

For a fixed location $z\in \cM$ and a fixed set $C$ of cluster centers, let $d(z, C) = \min_{c \in C} d(z,c)$ denote the distance of $z$ to its nearest cluster center in $C$. 
Then the cost of a solution $C$ on points $Z = (Z_1, Z_2, \ldots, Z_n)$ is
\[
\cost(C, Z)= \sum_{i\in [n], z\in \cZ} \bfo\{Z_i = z\}\cdot d(z, C).
\]
Thus, the expected cost of solution $C$ is given by
\[
\E[\cost(C,Z)]= \sum_{i\in [n], z\in \cZ} \Pr[Z_i = z]\cdot d(z, C),
\]
and $\opt$ denotes the expected cost of a set of cluster centers which minimizes this expected cost. 

In the bandit setting, we do not observe the realizations of the point locations. In each round $t$ a set of cluster center locations $C_t$ is submitted, after which the learner only observes  $\cost(C_t,Z)$. 

Let $\cN(C) = \{C'\in \cX \mid  C \cup \{b\} \setminus \{a\}, a \in C, b \in \cX\setminus C\}$ be the swap-neighborhood of cluster centers $C$.  
Arya et. al.~\cite{arya2001local} propose a local search   algorithm for the $k$-medians problem which starts from an arbitrary set of $k$ cluster centers $C$ and repeatedly moves to neighboring $C'\in \cN(C)$ satisfying $\cost(C') \leq (1-\epsilon) \cost(C)$ until no local improvement is found.    \citet{cormode2008approximation} shows that this is equally effective for the \emph{uncertain} $k$-medians problem. 

\begin{lemma}[\citet{arya2001local} and \citet{cormode2008approximation}] 
\label{kmedian-localsearch}
    For the uncertain $k$-median clustering problem, local search with $\epsilon$-improvements terminates in $C\in \cX$ satisfying $\cost(C)\leq (5/(1 - n\epsilon)) \cost(C^*).$  
\end{lemma} 

We can leverage this to show that uncertain $k$-median clustering permits $(\beta, \gamma)$-improving moves.  

\begin{lemma}
The Minimum Cost $K$-Median problem with $n$ input points admits $(\beta, \gamma)$-improving moves where $\gamma = 5/(1-\frac{1}{n})$ and $\beta = (1 - {1}/{n^2})$. 
\end{lemma}
\begin{proof}

We prove this lemma by contradiction. Suppose that the local search algorithm does not admit $(\beta, \gamma)$-improving moves where $\gamma = 5/(1-\frac{1}{n})$ and $\beta = (1 - {1}/{n^2})$. This implies that there exists $C\in \cX$ with $\cost(C) > \gamma \cost(C^*)$ and that $\cost(C') > \beta \cost(C)$ for all $C'\in\cN(C)$. 
It follows that the local search algorithm of \citep{cormode2008approximation} may terminate with  $C$ as a local optimum. 
 However, setting $\epsilon = 1/n^2$ in  \cref{kmedian-localsearch} gives   a $5/(1 - 1/n)$-factor approximation algorithm. 
 This implies that $\cost(C) \leq \frac{5}{1-\frac{1}{n}} \cost(C^*)$, a contradiction. 

\end{proof}

\begin{corollary} \label{corr_clustering}
    Consider the uncertain $k$-median clustering with candidate set of $m$ potential cluster centers where the locations of $n$ points in a  metric space of diameter 1 are drawn from distribution $\cD$. There exists algorithm $\cA$ (\cref{alg:eps_regret}) achieving
    \[
    \regret_{5n/(n-1)}(\cA_, \cD, T) 
    = O\left({n^{9}m^2\log^3 T}\right),
    \]
    when $m^2 \leq T$.
\end{corollary} 
\begin{proof}
    We apply \cref{thm:gamma_regret_main_result} and simplify with problem-specific parameters. For the swap neighborhood, the maximum neighborhood size is $M\leq |\cX|^2 = m^2$. The metric space has diameter 1, so  $C_{max} \leq n.$
    As in the proof of \Cref{corr_scheduling}, We can bound $   \log(1/\alpha)  \geq 1/{2n^2}$ and 
    $\delta^2 \beta   \geq 8/n^4$. The result follows after substitution and observing $\log n \leq \log T$.
\end{proof}

\section{Discussion} \label{sec:conclusion}
We conclude by discussing a couple of relevant issues not raised elsewhere. 

For ease of exposition we assumed the algorithm knows $T$, and that there is an upper bound $C_{\max}$ on costs. 
It is straightforward but tedious to modify the algorithms and analysis to not require explicit knowledge of $T$.  
Similarly, we can replace the assumption of bounded costs with suitable distributional assumptions about the cost.
For example, we could require a sub-gaussian like condition on the tails of the distribution or even just finite variance.
In the latter case, we could make use of tools for handling heavy-tailed distributions in the bandit setting~\cite{DBLP:journals/tit/BubeckCL13}. 

It remains to computationally demonstrate the benefits of our algorithm.  
All three  applications are instances of linear bandits\citep{DaniKH2007,DBLP:conf/www/LiCLS10,BubeckCK12}, so a natural starting point may be to compare the regret of our algorithm with that of a linear bandit style algorithm which, in each step, uses the bandit feedback to estimate problem parameters (eg.\ mean job sizes for completion time scheduling), then computes the solution to an offline instance of the problem with those parameters to submit in the next time step. For the scheduling and matroid applications the offline instances can be solved efficiently, for clustering we   have to resort to an approximate solution due to NP-hardness.

An exciting direction for future work is exploring further applications for our framework where local search is effective offline. 
Some applications may require generalizing our current approach, for example, the local search guarantees for set cover~\citep{DBLP:conf/sosa/GuptaLeeLi23} rely on finding neighboring solution which improves a specialized potential function, and not the objective function directly. 
It is an open question whether a version of $(\beta, \gamma)$-improving moves defined on such potential functions can  provide guarantees in the bandit online setting.  


\bibliographystyle{plainnat}

\bibliography{references.bib}

@inproceedings{DBLP:conf/www/LiCLS10,
  author       = {Lihong Li and
                  Wei Chu and
                  John Langford and
                  Robert E. Schapire},
  editor       = {Michael Rappa and
                  Paul Jones and
                  Juliana Freire and
                  Soumen Chakrabarti},
  title        = {A contextual-bandit approach to personalized news article recommendation},
  booktitle    = {Proceedings of the 19th International Conference on World Wide Web,
                  {WWW} 2010, Raleigh, North Carolina, USA, April 26-30, 2010},
  pages        = {661--670},
  publisher    = {{ACM}},
  year         = {2010},
  url          = {https://doi.org/10.1145/1772690.1772758},
  doi          = {10.1145/1772690.1772758},
  bibsource    = {dblp computer science bibliography, https://dblp.org}
}

@inproceedings{DBLP:conf/focs/AgarwalGN24,
  author       = {Arpit Agarwal and
                  Rohan Ghuge and
                  Viswanath Nagarajan},
  title        = {Semi-Bandit Learning for Monotone Stochastic Optimization},
  booktitle    = {65th {IEEE} Annual Symposium on Foundations of Computer Science, {FOCS}
                  2024, Chicago, IL, USA, October 27-30, 2024},
  pages        = {1260--1274},
  publisher    = {{IEEE}},
  year         = {2024},
  url          = {https://doi.org/10.1109/FOCS61266.2024.00083},
  doi          = {10.1109/FOCS61266.2024.00083},
  bibsource    = {dblp computer science bibliography, https://dblp.org}
}

@article{DBLP:journals/ior/VeraBG21,
  author       = {Alberto Vera and
                  Siddhartha Banerjee and
                  Itai Gurvich},
  title        = {Online Allocation and Pricing: Constant Regret via Bellman Inequalities},
  journal      = {Oper. Res.},
  volume       = {69},
  number       = {3},
  pages        = {821--840},
  year         = {2021},
  url          = {https://doi.org/10.1287/opre.2020.2061},
  doi          = {10.1287/OPRE.2020.2061},
  bibsource    = {dblp computer science bibliography, https://dblp.org}
}

@article{DBLP:journals/pomacs/CuvelierCG21,
  author       = {Thibaut Cuvelier and
                  Richard Combes and
                  Eric Gourdin},
  title        = {Statistically Efficient, Polynomial-Time Algorithms for Combinatorial
                  Semi-Bandits},
  journal      = {Proc. {ACM} Meas. Anal. Comput. Syst.},
  volume       = {5},
  number       = {1},
  pages        = {09:1--09:31},
  year         = {2021},
  url          = {https://doi.org/10.1145/3447387},
  doi          = {10.1145/3447387},
  bibsource    = {dblp computer science bibliography, https://dblp.org}
}

@inproceedings{DBLP:conf/pods/GuhaM09,
  author       = {Sudipto Guha and
                  Kamesh Munagala},
  editor       = {Jan Paredaens and
                  Jianwen Su},
  title        = {Exceeding expectations and clustering uncertain data},
  booktitle    = {Proceedings of the Twenty-Eigth {ACM} {SIGMOD-SIGACT-SIGART} Symposium
                  on Principles of Database Systems, {PODS} 2009, June 19 - July 1,
                  2009, Providence, Rhode Island, {USA}},
  pages        = {269--278},
  publisher    = {{ACM}},
  year         = {2009},
  url          = {https://doi.org/10.1145/1559795.1559836},
  doi          = {10.1145/1559795.1559836},
  bibsource    = {dblp computer science bibliography, https://dblp.org}
}

@inproceedings{DBLP:conf/sosa/GuptaLeeLi23, 
  author       = {Anupam Gupta and
                  Euiwoong Lee and
                  Jason Li},
  editor       = {Telikepalli Kavitha and
                  Kurt Mehlhorn},
  title        = {A Local Search-Based Approach for Set Covering},
  booktitle    = {2023 Symposium on Simplicity in Algorithms, {SOSA} 2023, Florence,
                  Italy, January 23-25, 2023},
  pages        = {1--11},
  publisher    = {{SIAM}},
  year         = {2023},
  url          = {https://doi.org/10.1137/1.9781611977585.ch1},
  doi          = {10.1137/1.9781611977585.CH1},
  bibsource    = {dblp computer science bibliography, https://dblp.org}
}

@article{DBLP:journals/tit/BubeckCL13,
  author       = {S{\'{e}}bastien Bubeck and
                  Nicol{\`{o}} Cesa{-}Bianchi and
                  G{\'{a}}bor Lugosi},
  title        = {Bandits With Heavy Tail},
  journal      = {{IEEE} Trans. Inf. Theory},
  volume       = {59},
  number       = {11},
  pages        = {7711--7717},
  year         = {2013},
  url          = {https://doi.org/10.1109/TIT.2013.2277869},
  doi          = {10.1109/TIT.2013.2277869},
  bibsource    = {dblp computer science bibliography, https://dblp.org}
}

@article{DBLP:journals/siamcomp/AuerCFS02,
  author       = {Peter Auer and
                  Nicol{\`{o}} Cesa{-}Bianchi and
                  Yoav Freund and
                  Robert E. Schapire},
  title        = {The Nonstochastic Multiarmed Bandit Problem},
  journal      = {{SIAM} J. Comput.},
  volume       = {32},
  number       = {1},
  pages        = {48--77},
  year         = {2002},
  url          = {https://doi.org/10.1137/S0097539701398375},
  doi          = {10.1137/S0097539701398375},
  bibsource    = {dblp computer science bibliography, https://dblp.org}
}

@article{LaiRobbins85,
author = {Lai, T.L and Robbins, Herbert},
title = {Asymptotically efficient adaptive allocation rules},
year = {1985},
issue_date = {March, 1985},
publisher = {Academic Press, Inc.},
address = {USA},
volume = {6},
number = {1},
issn = {0196-8858},
url = {https://doi.org/10.1016/0196-8858(85)90002-8},
doi = {10.1016/0196-8858(85)90002-8},
journal = {Adv. Appl. Math.},
month = mar,
pages = {4–22},
numpages = {19}
}

@inproceedings{DBLP:conf/nips/CombesSPL15,
  author       = {Richard Combes and
                  Mohammad Sadegh Talebi and
                  Alexandre Prouti{\`{e}}re and
                  Marc Lelarge},
  editor       = {Corinna Cortes and
                  Neil D. Lawrence and
                  Daniel D. Lee and
                  Masashi Sugiyama and
                  Roman Garnett},
  title        = {Combinatorial Bandits Revisited},
  booktitle    = {Advances in Neural Information Processing Systems 28: Annual Conference
                  on Neural Information Processing Systems 2015, December 7-12, 2015,
                  Montreal, Quebec, Canada},
  pages        = {2116--2124},
  year         = {2015},
  url          = {https://proceedings.neurips.cc/paper/2015/hash/0ce2ffd21fc958d9ef0ee9ba5336e357-Abstract.html},
  bibsource    = {dblp computer science bibliography, https://dblp.org}
}

@inproceedings{DBLP:conf/icml/ChenWY13,
  author       = {Wei Chen and
                  Yajun Wang and
                  Yang Yuan},
  title        = {Combinatorial Multi-Armed Bandit: General Framework and Applications},
  booktitle    = {Proceedings of the 30th International Conference on Machine Learning,
                  {ICML} 2013, Atlanta, GA, USA, 16-21 June 2013},
  series       = {{JMLR} Workshop and Conference Proceedings},
  volume       = {28},
  pages        = {151--159},
  publisher    = {JMLR.org},
  year         = {2013},
  url          = {http://proceedings.mlr.press/v28/chen13a.html},
  bibsource    = {dblp computer science bibliography, https://dblp.org}
}

@article{DBLP:journals/mor/AudibertBL14,
  author       = {Jean{-}Yves Audibert and
                  S{\'{e}}bastien Bubeck and
                  G{\'{a}}bor Lugosi},
  title        = {Regret in Online Combinatorial Optimization},
  journal      = {Math. Oper. Res.},
  volume       = {39},
  number       = {1},
  pages        = {31--45},
  year         = {2014},
  url          = {https://doi.org/10.1287/moor.2013.0598},
  doi          = {10.1287/MOOR.2013.0598},
  bibsource    = {dblp computer science bibliography, https://dblp.org}
}

@article{DBLP:journals/jacm/DuditeHLSSV20,
  author       = {Miroslav Dud{\'{\i}}k and
                  Nika Haghtalab and
                  Haipeng Luo and
                  Robert E. Schapire and
                  Vasilis Syrgkanis and
                  Jennifer Wortman Vaughan},
  title        = {Oracle-efficient Online Learning and Auction Design},
  journal      = {J. {ACM}},
  volume       = {67},
  number       = {5},
  pages        = {26:1--26:57},
  year         = {2020},
  url          = {https://doi.org/10.1145/3402203},
  doi          = {10.1145/3402203},
  bibsource    = {dblp computer science bibliography, https://dblp.org}
}

@article{benoist2011randomized,
  title={Randomized local search for real-life inventory routing},
  author={Benoist, Thierry and Gardi, Fr{\'e}d{\'e}ric and Jeanjean, Antoine and Estellon, Bertrand},
  journal={Transportation Science},
  volume={45},
  number={3},
  pages={381--398},
  year={2011},
  publisher={INFORMS}
}

@article{musliu2004local,
  title={Local search for shift design},
  author={Musliu, Nysret and Schaerf, Andrea and Slany, Wolfgang},
  journal={European journal of operational research},
  volume={153},
  number={1},
  pages={51--64},
  year={2004},
  publisher={Elsevier}
}

@article{ceschia2013local,
  title={Local search for a multi-drop multi-container loading problem},
  author={Ceschia, Sara and Schaerf, Andrea},
  journal={Journal of Heuristics},
  volume={19},
  number={2},
  pages={275--294},
  year={2013},
  publisher={Springer}
}

@inproceedings{DBLP:conf/ijcai/Zheng0Z0LM22,
  author       = {Jiongzhi Zheng and
                  Kun He and
                  Jianrong Zhou and
                  Yan Jin and
                  Chu{-}Min Li and
                  Felip Many{\`{a}}},
  editor       = {Luc De Raedt},
  title        = {BandMaxSAT: {A} Local Search MaxSAT Solver with Multi-armed Bandit},
  booktitle    = {Proceedings of the Thirty-First International Joint Conference on
                  Artificial Intelligence, {IJCAI} 2022, Vienna, Austria, 23-29 July
                  2022},
  pages        = {1901--1907},
  publisher    = {ijcai.org},
  year         = {2022},
  url          = {https://doi.org/10.24963/ijcai.2022/264},
  doi          = {10.24963/IJCAI.2022/264},
  bibsource    = {dblp computer science bibliography, https://dblp.org}
}

@inproceedings{DBLP:conf/ecai/Mengshoel0Z20,
  author       = {Ole Jakob Mengshoel and
                  Tong Yu and
                  Ming Zeng},
  editor       = {Giuseppe De Giacomo and
                  Alejandro Catal{\'{a}} and
                  Bistra Dilkina and
                  Michela Milano and
                  Sen{\'{e}}n Barro and
                  Alberto Bugar{\'{\i}}n and
                  J{\'{e}}r{\^{o}}me Lang},
  title        = {Stochastic Local Search and Machine Learning: From Theory to Applications
                  and Vice Versa},
  booktitle    = {{ECAI} 2020 - 24th European Conference on Artificial Intelligence,
                  29 August-8 September 2020, Santiago de Compostela, Spain, August
                  29 - September 8, 2020 - Including 10th Conference on Prestigious
                  Applications of Artificial Intelligence {(PAIS} 2020)},
  series       = {Frontiers in Artificial Intelligence and Applications},
  volume       = {325},
  pages        = {2919--2920},
  publisher    = {{IOS} Press},
  year         = {2020},
  url          = {https://doi.org/10.3233/FAIA200453},
  doi          = {10.3233/FAIA200453},
  bibsource    = {dblp computer science bibliography, https://dblp.org}
}

@article{DBLP:journals/eor/LagosP24,
  author       = {Felipe Lagos and
                  Jordi Pereira},
  title        = {Multi-armed bandit-based hyper-heuristics for combinatorial optimization
                  problems},
  journal      = {Eur. J. Oper. Res.},
  volume       = {312},
  number       = {1},
  pages        = {70--91},
  year         = {2024},
  url          = {https://doi.org/10.1016/j.ejor.2023.06.016},
  doi          = {10.1016/J.EJOR.2023.06.016},
  bibsource    = {dblp computer science bibliography, https://dblp.org}
}

@inproceedings{DBLP:conf/pkdd/YuKM17,
  author       = {Tong Yu and
                  Branislav Kveton and
                  Ole J. Mengshoel},
  editor       = {Michelangelo Ceci and
                  Jaakko Hollm{\'{e}}n and
                  Ljupco Todorovski and
                  Celine Vens and
                  Saso Dzeroski},
  title        = {Thompson Sampling for Optimizing Stochastic Local Search},
  booktitle    = {Machine Learning and Knowledge Discovery in Databases - European Conference,
                  {ECML} {PKDD} 2017, Skopje, Macedonia, September 18-22, 2017, Proceedings,
                  Part {I}},
  series       = {Lecture Notes in Computer Science},
  volume       = {10534},
  pages        = {493--510},
  publisher    = {Springer},
  year         = {2017},
  url          = {https://doi.org/10.1007/978-3-319-71249-9\_30},
  doi          = {10.1007/978-3-319-71249-9\_30},
  bibsource    = {dblp computer science bibliography, https://dblp.org}
}

@inproceedings{DBLP:conf/nips/XuW21,
  author       = {Jianyu Xu and
                  Yu{-}Xiang Wang},
  editor       = {Marc'Aurelio Ranzato and
                  Alina Beygelzimer and
                  Yann N. Dauphin and
                  Percy Liang and
                  Jennifer Wortman Vaughan},
  title        = {Logarithmic Regret in Feature-based Dynamic Pricing},
  booktitle    = {Advances in Neural Information Processing Systems 34: Annual Conference
                  on Neural Information Processing Systems 2021, NeurIPS 2021, December
                  6-14, 2021, virtual},
  pages        = {13898--13910},
  year         = {2021},
  url          = {https://proceedings.neurips.cc/paper/2021/hash/742141ceda6b8f6786609d31c8ef129f-Abstract.html},
  bibsource    = {dblp computer science bibliography, https://dblp.org}
}

@article{DBLP:journals/ml/HazanAK07,
  author       = {Elad Hazan and
                  Amit Agarwal and
                  Satyen Kale},
  title        = {Logarithmic regret algorithms for online convex optimization},
  journal      = {Mach. Learn.},
  volume       = {69},
  number       = {2-3},
  pages        = {169--192},
  year         = {2007},
  url          = {https://doi.org/10.1007/s10994-007-5016-8},
  doi          = {10.1007/S10994-007-5016-8},
  bibsource    = {dblp computer science bibliography, https://dblp.org}
}

@inproceedings{DBLP:conf/nips/Shalev-ShwartzK08,
  author       = {Shai Shalev{-}Shwartz and
                  Sham M. Kakade},
  editor       = {Daphne Koller and
                  Dale Schuurmans and
                  Yoshua Bengio and
                  L{\'{e}}on Bottou},
  title        = {Mind the Duality Gap: Logarithmic regret algorithms for online optimization},
  booktitle    = {Advances in Neural Information Processing Systems 21, Proceedings
                  of the Twenty-Second Annual Conference on Neural Information Processing
                  Systems, Vancouver, British Columbia, Canada, December 8-11, 2008},
  pages        = {1457--1464},
  publisher    = {Curran Associates, Inc.},
  year         = {2008},
  url          = {https://proceedings.neurips.cc/paper/2008/hash/bd686fd640be98efaae0091fa301e613-Abstract.html},
  bibsource    = {dblp computer science bibliography, https://dblp.org}
}

@article{DBLP:journals/siamcomp/KakadeKL09,
  author       = {Sham M. Kakade and
                  Adam Tauman Kalai and
                  Katrina Ligett},
  title        = {Playing Games with Approximation Algorithms},
  journal      = {{SIAM} J. Comput.},
  volume       = {39},
  number       = {3},
  pages        = {1088--1106},
  year         = {2009},
  url          = {https://doi.org/10.1137/070701704},
  doi          = {10.1137/070701704},
  bibsource    = {dblp computer science bibliography, https://dblp.org}
}

@article{DBLP:journals/jcss/KalaiV05,
  author       = {Adam Tauman Kalai and
                  Santosh S. Vempala},
  title        = {Efficient algorithms for online decision problems},
  journal      = {J. Comput. Syst. Sci.},
  volume       = {71},
  number       = {3},
  pages        = {291--307},
  year         = {2005},
  url          = {https://doi.org/10.1016/j.jcss.2004.10.016},
  doi          = {10.1016/J.JCSS.2004.10.016},
  bibsource    = {dblp computer science bibliography, https://dblp.org}
}

@inproceedings{Cohen-AddadGHOS22,
  author       = {Vincent Cohen{-}Addad and
                  Anupam Gupta and
                  Lunjia Hu and
                  Hoon Oh and
                  David Saulpic},
  editor       = {Joseph (Seffi) Naor and
                  Niv Buchbinder},
  title        = {An Improved Local Search Algorithm for k-Median},
  booktitle    = {Proceedings of the 2022 {ACM-SIAM} Symposium on Discrete Algorithms,
                  {SODA} 2022, Virtual Conference / Alexandria, VA, USA, January 9 -
                  12, 2022},
  pages        = {1556--1612},
  publisher    = {{SIAM}},
  year         = {2022},
  url          = {https://doi.org/10.1137/1.9781611977073.65},
  doi          = {10.1137/1.9781611977073.65},
  bibsource    = {dblp computer science bibliography, https://dblp.org}
}

@inproceedings{AhmadianFS13,
  author       = {Sara Ahmadian and
                  Zachary Friggstad and
                  Chaitanya Swamy},
  editor       = {Sanjeev Khanna},
  title        = {Local-Search based Approximation Algorithms for Mobile Facility Location
                  Problems},
  booktitle    = {Proceedings of the Twenty-Fourth Annual {ACM-SIAM} Symposium on Discrete
                  Algorithms, {SODA} 2013, New Orleans, Louisiana, USA, January 6-8,
                  2013},
  pages        = {1607--1621},
  publisher    = {{SIAM}},
  year         = {2013},
  url          = {https://doi.org/10.1137/1.9781611973105.115},
  doi          = {10.1137/1.9781611973105.115},
  bibsource    = {dblp computer science bibliography, https://dblp.org}
}

@inproceedings{Chen08,
  author       = {Ke Chen},
  editor       = {Shang{-}Hua Teng},
  title        = {A constant factor approximation algorithm for \emph{k}-median clustering
                  with outliers},
  booktitle    = {Proceedings of the Nineteenth Annual {ACM-SIAM} Symposium on Discrete
                  Algorithms, {SODA} 2008, San Francisco, California, USA, January 20-22,
                  2008},
  pages        = {826--835},
  publisher    = {{SIAM}},
  year         = {2008},
  url          = {http://dl.acm.org/citation.cfm?id=1347082.1347173},
  bibsource    = {dblp computer science bibliography, https://dblp.org}
}

@inproceedings{BogdanSS013,
  author       = {Paul Bogdan and
                  Thomas Sauerwald and
                  Alexandre Stauffer and
                  He Sun},
  editor       = {Sanjeev Khanna},
  title        = {Balls into Bins via Local Search},
  booktitle    = {Proceedings of the Twenty-Fourth Annual {ACM-SIAM} Symposium on Discrete
                  Algorithms, {SODA} 2013, New Orleans, Louisiana, USA, January 6-8,
                  2013},
  pages        = {16--34},
  publisher    = {{SIAM}},
  year         = {2013},
  url          = {https://doi.org/10.1137/1.9781611973105.2},
  doi          = {10.1137/1.9781611973105.2},
  bibsource    = {dblp computer science bibliography, https://dblp.org}
}

@inproceedings{IwamaMY07,
  author       = {Kazuo Iwama and
                  Shuichi Miyazaki and
                  Naoya Yamauchi},
  editor       = {Nikhil Bansal and
                  Kirk Pruhs and
                  Clifford Stein},
  title        = {A 1.875: approximation algorithm for the stable marriage problem},
  booktitle    = {Proceedings of the Eighteenth Annual {ACM-SIAM} Symposium on Discrete
                  Algorithms, {SODA} 2007, New Orleans, Louisiana, USA, January 7-9,
                  2007},
  pages        = {288--297},
  publisher    = {{SIAM}},
  year         = {2007},
  url          = {http://dl.acm.org/citation.cfm?id=1283383.1283414},
  bibsource    = {dblp computer science bibliography, https://dblp.org}
}

@inproceedings{GharanT14,
  author       = {Shayan Oveis Gharan and
                  Luca Trevisan},
  editor       = {Chandra Chekuri},
  title        = {Partitioning into Expanders},
  booktitle    = {Proceedings of the Twenty-Fifth Annual {ACM-SIAM} Symposium on Discrete
                  Algorithms, {SODA} 2014, Portland, Oregon, USA, January 5-7, 2014},
  pages        = {1256--1266},
  publisher    = {{SIAM}},
  year         = {2014},
  url          = {https://doi.org/10.1137/1.9781611973402.93},
  doi          = {10.1137/1.9781611973402.93},
  bibsource    = {dblp computer science bibliography, https://dblp.org}
}

@inproceedings{WenKA15,
  author       = {Zheng Wen and
                  Branislav Kveton and
                  Azin Ashkan},
  editor       = {Francis R. Bach and
                  David M. Blei},
  title        = {Efficient Learning in Large-Scale Combinatorial Semi-Bandits},
  booktitle    = {Proceedings of the 32nd International Conference on Machine Learning,
                  {ICML} 2015, Lille, France, 6-11 July 2015},
  series       = {{JMLR} Workshop and Conference Proceedings},
  volume       = {37},
  pages        = {1113--1122},
  publisher    = {JMLR.org},
  year         = {2015},
  url          = {http://proceedings.mlr.press/v37/wen15.html},
  bibsource    = {dblp computer science bibliography, https://dblp.org}
}

@article{NeuB16,
  author       = {Gergely Neu and
                  G{\'{a}}bor Bart{\'{o}}k},
  title        = {Importance Weighting Without Importance Weights: An Efficient Algorithm
                  for Combinatorial Semi-Bandits},
  journal      = {J. Mach. Learn. Res.},
  volume       = {17},
  pages        = {154:1--154:21},
  year         = {2016},
  url          = {https://jmlr.org/papers/v17/15-091.html},
  bibsource    = {dblp computer science bibliography, https://dblp.org}
}

@inproceedings{PerraultPV19,
  author       = {Pierre Perrault and
                  Vianney Perchet and
                  Michal Valko},
  editor       = {Kamalika Chaudhuri and
                  Ruslan Salakhutdinov},
  title        = {Exploiting structure of uncertainty for efficient matroid semi-bandits},
  booktitle    = {Proceedings of the 36th International Conference on Machine Learning,
                  {ICML} 2019, 9-15 June 2019, Long Beach, California, {USA}},
  series       = {Proceedings of Machine Learning Research},
  volume       = {97},
  pages        = {5123--5132},
  publisher    = {{PMLR}},
  year         = {2019},
  url          = {http://proceedings.mlr.press/v97/perrault19a.html},
  bibsource    = {dblp computer science bibliography, https://dblp.org}
}

@article{NiazadehGWSB_MS23,
  author       = {Rad Niazadeh and
                  Negin Golrezaei and
                  Joshua R. Wang and
                  Fransisca Susan and
                  Ashwinkumar Badanidiyuru},
  title        = {Online Learning via Offline Greedy Algorithms: Applications in Market
                  Design and Optimization},
  journal      = {Manag. Sci.},
  volume       = {69},
  number       = {7},
  pages        = {3797--3817},
  year         = {2023},
  url          = {https://doi.org/10.1287/mnsc.2022.4558},
  doi          = {10.1287/MNSC.2022.4558},
  bibsource    = {dblp computer science bibliography, https://dblp.org}
}

@inproceedings{FouratiQAA24,
  author       = {Fares Fourati and
                  Christopher John Quinn and
                  Mohamed{-}Slim Alouini and
                  Vaneet Aggarwal},
  editor       = {Michael J. Wooldridge and
                  Jennifer G. Dy and
                  Sriraam Natarajan},
  title        = {Combinatorial Stochastic-Greedy Bandit},
  booktitle    = {Thirty-Eighth {AAAI} Conference on Artificial Intelligence, {AAAI}
                  2024, Thirty-Sixth Conference on Innovative Applications of Artificial
                  Intelligence, {IAAI} 2024, Fourteenth Symposium on Educational Advances
                  in Artificial Intelligence, {EAAI} 2014, February 20-27, 2024, Vancouver,
                  Canada},
  pages        = {12052--12060},
  publisher    = {{AAAI} Press},
  year         = {2024},
  url          = {https://doi.org/10.1609/aaai.v38i11.29093},
  doi          = {10.1609/AAAI.V38I11.29093},
  bibsource    = {dblp computer science bibliography, https://dblp.org}
}

@book{Dubhashi_Panconesi_2009, 
    location={Cambridge},
    title={Concentration of Measure for the Analysis of Randomized Algorithms},
    publisher={Cambridge University Press}, 
    author={Dubhashi, Devdatt P. and Panconesi, Alessandro}, 
    year={2009},
}

@article{DBLP:journals/jmlr/Even-DarMM06,
  author       = {Eyal Even{-}Dar and
                  Shie Mannor and
                  Yishay Mansour},
  title        = {Action Elimination and Stopping Conditions for the Multi-Armed Bandit
                  and Reinforcement Learning Problems},
  journal      = {J. Mach. Learn. Res.},
  volume       = {7},
  pages        = {1079--1105},
  year         = {2006},
  url          = {https://jmlr.org/papers/v7/evendar06a.html},
  bibsource    = {dblp computer science bibliography, https://dblp.org}
}

@inproceedings{DBLP:conf/colt/Even-DarMM02,
  author       = {Eyal Even{-}Dar and
                  Shie Mannor and
                  Yishay Mansour},
  editor       = {Jyrki Kivinen and
                  Robert H. Sloan},
  title        = {{PAC} Bounds for Multi-armed Bandit and Markov Decision Processes},
  booktitle    = {Computational Learning Theory, 15th Annual Conference on Computational
                  Learning Theory, {COLT} 2002, Sydney, Australia, July 8-10, 2002,
                  Proceedings},
  series       = {Lecture Notes in Computer Science},
  volume       = {2375},
  pages        = {255--270},
  publisher    = {Springer},
  year         = {2002},
  url          = {https://doi.org/10.1007/3-540-45435-7\_18},
  doi          = {10.1007/3-540-45435-7\_18},
  bibsource    = {dblp computer science bibliography, https://dblp.org}
}

@inproceedings{BubeckCK12,
  author       = {S{\'{e}}bastien Bubeck and
                  Nicol{\`{o}} Cesa{-}Bianchi and
                  Sham M. Kakade},
  editor       = {Shie Mannor and
                  Nathan Srebro and
                  Robert C. Williamson},
  title        = {Towards Minimax Policies for Online Linear Optimization with Bandit
                  Feedback},
  booktitle    = {{COLT} 2012 - The 25th Annual Conference on Learning Theory, June
                  25-27, 2012, Edinburgh, Scotland},
  series       = {{JMLR} Proceedings},
  volume       = {23},
  pages        = {41.1--41.14},
  publisher    = {JMLR.org},
  year         = {2012},
  url          = {http://proceedings.mlr.press/v23/bubeck12a/bubeck12a.pdf},
  bibsource    = {dblp computer science bibliography, https://dblp.org}
}

@inproceedings{DaniKH2007,
 author = {Dani, Varsha and Kakade, Sham M and Hayes, Thomas},
 booktitle = {Advances in Neural Information Processing Systems},
 pages = {},
 publisher = {Curran Associates, Inc.},
 title = {The Price of Bandit Information for Online Optimization},
 url = {https://proceedings.neurips.cc/paper_files/paper/2007/file/bf62768ca46b6c3b5bea9515d1a1fc45-Paper.pdf},
 volume = {20},
 year = {2007}
}

@article{Smith1956,
author = {Smith, Wayne E.},
title = {Various optimizers for single-stage production},
journal = {Naval Research Logistics Quarterly},
volume = {3},
number = {1-2},
pages = {59-66},
doi = {https://doi.org/10.1002/nav.3800030106},
url = {https://onlinelibrary.wiley.com/doi/abs/10.1002/nav.3800030106},
eprint = {https://onlinelibrary.wiley.com/doi/pdf/10.1002/nav.3800030106},
year = {1956}
}

@article{Rothkopf1966,
 ISSN = {00251909, 15265501},
 URL = {http://www.jstor.org/stable/2627947},
 author = {Michael H. Rothkopf},
 journal = {Management Science},
 number = {9},
 pages = {707--713},
 publisher = {INFORMS},
 title = {Scheduling with Random Service Times},
 urldate = {2024-08-13},
 volume = {12},
 year = {1966}
}

@misc{glasgow2023tight,
      title={Tight Bounds for $\gamma$-Regret via the Decision-Estimation Coefficient}, 
      author={Margalit Glasgow and Alexander Rakhlin},
      year={2023},
      eprint={2303.03327},
      archivePrefix={arXiv},
      primaryClass={cs.LG}
}

@book{lattimore2020bandit, 
    place={Cambridge}, 
    title={Bandit Algorithms}, 
    publisher={Cambridge University Press}, 
    author={Lattimore, Tor and Szepesvári, Csaba}, 
    year={2020}
}

@inproceedings{DBLP:conf/spaa/LindermayrM22,
  author       = {Alexander Lindermayr and
                  Nicole Megow},
  editor       = {Kunal Agrawal and
                  I{-}Ting Angelina Lee},
  title        = {Permutation Predictions for Non-Clairvoyant Scheduling},
  booktitle    = {{SPAA} '22: 34th {ACM} Symposium on Parallelism in Algorithms and
                  Architectures, Philadelphia, PA, USA, July 11 - 14, 2022},
  pages        = {357--368},
  publisher    = {{ACM}},
  year         = {2022},
  url          = {https://doi.org/10.1145/3490148.3538579},
  doi          = {10.1145/3490148.3538579},
  bibsource    = {dblp computer science bibliography, https://dblp.org}
}

@inproceedings{DBLP:conf/icml/YangRS0DS22,
  author       = {Shuo Yang and
                  Tongzheng Ren and
                  Sanjay Shakkottai and
                  Eric Price and
                  Inderjit S. Dhillon and
                  Sujay Sanghavi}, 
  title        = {Linear Bandit Algorithms with Sublinear Time Complexity},
  booktitle    = {International Conference on Machine Learning, {ICML} 2022, 17-23 July
                  2022, Baltimore, Maryland, {USA}},
  series       = {Proceedings of Machine Learning Research},
  volume       = {162},
  pages        = {25241--25260},
  publisher    = {{PMLR}},
  year         = {2022},
  url          = {https://proceedings.mlr.press/v162/yang22m.html},
  bibsource    = {dblp computer science bibliography, https://dblp.org}
}

@article{DBLP:journals/ml/AuerCF02,
  author       = {Peter Auer and
                  Nicol{\`{o}} Cesa{-}Bianchi and
                  Paul Fischer},
  title        = {Finite-time Analysis of the Multiarmed Bandit Problem},
  journal      = {Mach. Learn.},
  volume       = {47},
  number       = {2-3},
  pages        = {235--256},
  year         = {2002},
  url          = {https://doi.org/10.1023/A:1013689704352},
  doi          = {10.1023/A:1013689704352},
  bibsource    = {dblp computer science bibliography, https://dblp.org}
}

@book{frank2011connections,
  title={Connections in combinatorial optimization},
  author={Frank, Andr{\'a}s},
  volume={38},
  year={2011},
  publisher={Oxford University Press Oxford}
}

@inproceedings{arya2001local,
  title={Local search heuristic for k-median and facility location problems},
  author={Arya, Vijay and Garg, Naveen and Khandekar, Rohit and Meyerson, Adam and Munagala, Kamesh and Pandit, Vinayaka},
  booktitle={Proceedings of the thirty-third annual ACM symposium on Theory of computing},
  pages={21--29},
  year={2001}
}

@inproceedings{cormode2008approximation,
  title={Approximation algorithms for clustering uncertain data},
  author={Cormode, Graham and McGregor, Andrew},
  booktitle={Proceedings of the twenty-seventh ACM SIGMOD-SIGACT-SIGART symposium on Principles of database systems},
  pages={191--200},
  year={2008}
}

@inproceedings{DBLP:conf/spaa/BenadeDL25,
  author       = {Gerdus Benad{\`{e}} and
                  Rathish Das and
                  Thomas Lavastida},
  title        = {Brief Announcement: Stochastic Parallel Scheduling with Bandit Feedback},
  booktitle    = {Proceedings of the 37th {ACM} Symposium on Parallelism in Algorithms
                  and Architectures, {SPAA} 2025, Portland, OR, USA, 28 July 2025 -
                  1 August 2025},
  pages        = {618--622},
  publisher    = {{ACM}},
  year         = {2025},
  url          = {https://doi.org/10.1145/3694906.3743344},
  doi          = {10.1145/3694906.3743344},
  bibsource    = {dblp computer science bibliography, https://dblp.org}
}

\end{document}